\documentclass[runningheads,a4paper,11pt]{article}
\usepackage{algorithmicx}
\usepackage[ruled]{algorithm2e}

\usepackage{amsmath,amsfonts,amssymb,bbm,amsthm,bm} 
\usepackage[numbers,sort&compress]{natbib} 

\usepackage[T1]{fontenc}
\usepackage[utf8]{inputenc}
\usepackage{authblk}

\usepackage{enumerate}
\usepackage{epsfig}
\usepackage{amsmath}
\usepackage{amstext}
\usepackage{eucal}
\usepackage{graphicx}
\usepackage{colordvi}
\usepackage{subfigure}
\usepackage{url}
\usepackage{color}
\usepackage{bbm}
\usepackage{capt-of}
\usepackage{url}

\newcommand{\E}{\textup{E}}

\allowdisplaybreaks
\usepackage{mathtools}

\DeclarePairedDelimiter\floor{\lfloor}{\rfloor}
\usepackage{enumitem}

\newtheorem{definition}{Definition}[section]
\newtheorem{theorem}{Theorem}
\newtheorem{lemma}[theorem]{Lemma}
\newtheorem{proposition}[theorem]{Proposition}

\newtheorem{comment}{Comment}

\usepackage{tikz}
\usetikzlibrary{graphs,graphs.standard}
\usetikzlibrary[topaths]

\title{Exploration vs. Exploitation in Team Formation for Collaborative Work}
\author{ Ramesh Johari$^1$,      Vijay Kamble$^2$,      Anilesh K. Krishnaswamy$^1$,   Hannah Li$^1$ \\
$^1$Stanford University, \{rjohari, anilesh, hannahli\}@stanford.edu\\
$^2$University of Illinois at Chicago, kamble@uic.edu}
\date{}
\begin{document}
\maketitle

\begin{abstract}
Modern labor platforms face the online learning problem of optimizing matches between jobs and workers of unknown abilities. This problem is complicated by the rise of complex jobs on these platforms that require teamwork, such as web development and product design. Successful completion of such a job depends on the abilities of all workers involved, which can only be indirectly inferred by observing the aggregate performance of the team. 
Observations of the performance of various overlapping teams induce correlations between the unknown abilities of different workers at any given time. Tracking the evolution of this correlation structure across a large number of workers on the platform as new observations become available, and using this information to adaptively optimize future matches, is a challenging problem. 

 To study this problem, we develop a stylized model in which teams are of size 2 and each worker is drawn i.i.d. from a binary (good or bad) type distribution. Under this model, we analyze two natural settings: when the performance of a team is dictated by its strongest member and when it is dictated by its weakest member. We find that these two settings exhibit stark differences in the trade-offs between exploration (i.e., learning the performance of untested teams) and exploitation (i.e., repeating previously tested teams that resulted in a good performance). We establish fundamental regret bounds and design near-optimal algorithms that uncover several insights into these tradeoffs.
\end{abstract}

\section{Introduction}


Online labor market platforms (e.g., Upwork, Freelancer, and Amazon Mechanical Turk for online work; and Thumbtack and TaskRabbit for offline work) provide the ability for clients to hire freelancers through the platform on a task-by-task basis.  These platforms use data collected from past match outcomes to learn about freelancers, and improve matching in the future.  These platforms are primarily {\em one-to-one} matching platforms, i.e., clients hire one freelancer at a time.  However, for many tasks, clients need more than one freelancer; in such cases, they are left to form a team on their own -- platforms do not generally provide algorithmic team matching services.  
As organization structures such as ``flash teams'' and ``flash organizations'' begin to be enabled by online platforms \cite{retelny2014expert,valentine2017flash}, it will be increasingly incumbent on the platform to ensure it can optimally match teams of available freelancers to jobs at hand.

Existing research in the area of workforce utilization has focused on designing effective roles for workers and organization structures that streamline collaboration between individuals with different types of skills (``horizontal'' differentiation) \cite{talaulicar2005strategic,erb1989team}. The problem of matching teams of workers to available jobs in the face of quality uncertainty (``vertical'' differentiation), however, has largely remained unaddressed. In this paper, we bring this problem into focus.

Without any assumptions on how the composite performance of a team depends on the quality of its constituent individuals, the problem of learning the optimal composition of teams is a combinatorial search problem over all possible team combinations, and is intractable in many situations \cite{rajkumar2017learning}. In practice, it is natural to assume that the performance of a team depends in some form on the individual qualities of workers that constitute the team \cite{neuman1999relationship}. It is also natural to assume that the performance of a team can be evaluated based on the quantitative feedback received from the employer on completion of the assigned task, as is commonly seen in many online platforms.  

Such quantitative evaluation of the performance of a team reveals some information about the inherent qualities of constituent workers, and thereby the performance of other potential teams that contain some of these individuals. These learning spillovers depend on how the characteristics of workers combine to determine the performance of the team, and they can be effectively utilized to reduce the number of attempts required to obtain a good partitioning of the workers. However, optimally leveraging these spillovers requires the platform to solve the challenging problem of designing an intelligent and adaptive sequence of matches that learns at the expense of minimal loss in performance. 



 In this paper, we develop a stylized model to investigate this challenge; in the model we consider, the platform only observes a single aggregate outcome of each team's performance on a job. Thus identifying the quality of each worker requires observing their performance in distinct teams across {\em multiple jobs}.  Our model, despite its simplicity, shows surprisingly intricate structure, with the emergence of novel exploration-exploitation tradeoffs that shed light on the qualitative features of optimal matching policies in these settings.

The model we consider consists of a large number of workers.  Each worker is of one of two types: ``high'' (labeled ``1'') or ``low'' (labeled ``0''); we assume each worker is independently 1 with probability $p$.  We consider a model in discrete time, where at each time workers are matched into {\em pairs} (i.e., teams of size two) to complete jobs.  We consider two models for payoffs: the {\em Weakest Link} model, where the payoff of a team is the {\em minimum} of their types; and the {\em Strongest Link} model, where the payoff of a team is the {\em maximum} of their types.  The platform is able to observe the payoff from each pair, but not the types of individual workers; these must be inferred from the sequence of team outcomes obtained.  The goal of the platform is to maximize payoffs, so it must use the team matching to learn about workers while minimizing loss in payoff, or \emph{regret}.

The Weakest Link model is most natural when every worker's output is essential to completing the task successfully.  As an example, suppose two workers are hired to complete a web development job: one may need to complete front end development, and another back end development.  If either fails, the job itself is a failure.  
Our model assumes that assessment of the task is based on the final outcome (i.e., whether the site is functional), without attribution to the individual workers. Another example is house cleaning: if any cleaner performs her set of cleaning tasks poorly, the entire cleaning job may receive a poor rating.
The Strongest Link model is most natural when a strong worker can cover for shortcomings of her partner. This is a more natural representation of tasks that have planning and execution components, where high expertise is needed for effective planning, although a single high quality worker suffices for this, and the execution is less sensitive to skill.

The goal of the platform is to adaptively match the workforce into pairs at each stage so as to maximize the expected long-run payoff. In contrast to many online learning problems in which it is not possible to stop incurring regret after any finite time, in this problem any good policy will, in finite time, gather enough information to be able to make optimal matches (for example, one way to do this is to make all the possible matches in an initial phase, then choose an optimal partition going forward).   We therefore measure performance by computing \emph{cumulative} regret against a policy that knows all the worker types to begin with, until such a time when no additional regret is incurred thereafter.  Our main contributions involve an analysis of this regret in the Weakest Link and Strongest Link models.  We now summarize our main contributions.

{\em Weakest Link model}.  For the Weakest Link model, informally, the goal is to quickly identify all the 1 workers so that they can be matched to each other, while minimizing the number of $0-1$ matches in the process. The only way one can discover a 1 worker is if she gets matched to another 1 worker, resulting in a payoff of $1$. On the other hand, the only way to discover a 0 worker is if she gets matched to a worker who is known to be of a 1, or is later discovered to be a 1. Thus matching a worker with unknown quality to a 1 worker, either known or unknown, is critical to learning. 
This brings us to the two central questions. First, suppose that the platform has discovered a certain number of 1 workers. Should these workers be matched amongst themselves to generate high feedback (``exploit'') or should these workers be matched to workers whose qualities are unknown so as to speedup learning (``explore'')? Second, even if we want to only ``exploit'' and not ``explore" with known 1 workers, i.e., match the known 1 workers amongst themselves, what is the best way to adaptively match the unknown pool of workers amongst themselves? 

It turns out that the answers to these questions depends on the expected proportion of high quality workers in the population $p$. As $p\rightarrow 1$, i.e., when one expects there to be an abundance of 1 workers in the population, we construct a policy for matching unknown workers that we call \textsf{Exponential cliques}, that is asymptotically optimal (in $p$), without requiring the discovered 1 workers to explore. That is, the 1 workers can be simply matched amongst themselves after identification. This is a bit counterintuitive because in this regime, the number of 1 workers that get identified in the first stage is high enough to be able to use these workers to learn the quality of \emph{all} the remaining workers in the second stage. One incurs no regret thereafter. Although this type of an approach appears to be tempting from an operational perspective -- it is certainly fastest in terms of learning -- we nevertheless show that it is strictly suboptimal. In stark contrast, for $p$ low enough, we show that any optimal matching policy \emph{must} match discovered 1 workers with workers of unknown quality. 

{\em Strongest Link model}.  For the Strongest Link model, informally, the goal is to match 1 workers with 0 workers to the extent possible. If the number of workers is large and $p<0.5$, then some $0-0$ matches are inevitable, and thus one wants to minimize the number of $1-1$ matches.  On the other hand, if $p>0.5$, then some $1-1$ matches are inevitable, and thus one wants to minimize the number of $0-0$ matches. 
In either case, 0 workers get discovered when they are matched to other 0 workers, whereas 1 workers are identified by either being matched to a known 0 worker, or to an unknown 0 worker who later gets identified as being a 0. 
Thus, matching a worker with unknown quality to a 0 worker, either known or unknown, is critical to learning. But in contrast to the Weakest Link model, the question of what is to be done with the discovered 0 workers is less uncertain. Though it may seem natural at first glance to not match known 0 workers to anyone but each other, it turns out that it is strictly better to utilize these known-quality workers to explore the unknown worker set. The central question then becomes: how does one optimally explore using these known 0 workers? Restricting ourselves to a natural class of exploration policies, we uncover a sharp transition in the structure of the optimal exploration policy at $p = 0.5$. 

{\em Organization of the paper}.  The remainder of the paper is organized as follows.  We discuss relevant literature in Section~\ref{sec:lit}. We introduce the model, the problem formulation and certain reductions of the problem in Section~\ref{sec:model}. In Sections~\ref{sec:min} and~\ref{sec:max}, we focus on the Weakest Link and the Strongest Link models, respectively. We conclude with a discussion of our results in~\ref{sec:conclusion}. The proofs of all of our results are deferred to the appendix.

\section{Related Work}\label{sec:lit}



In this paper, we focus on optimizing platform performance while learning to form optimal teams, under the presence of uncertainty about worker qualifications and ultimately each worker's contribution to a job. There has been some work on the problem of optimal team formation, not emphasizing the aspects of learning or optimization, but the incentivization of workers by the platform to participate or exert effort \cite{babaioff2006combinatorial,carlier2010matching}. This line of work uses game theoretic models to analyze optimal aggregate performance when the platform can see only the project outcome and not the individual contributions. The platform can offer prices or contracts to strategic agents to motivate them to work in teams. 

Kleinberg and Raghu \cite{kleinberg2015team} propose a method of learning which individuals, among a pool of candidates, will together form the best team. This method shows that even when a project depends on interactions between the members and the performance is a complex function of the particular subset chosen, in certain cases near-optimal teams can be formed by looking at individual performance scores. This work assumes the ability to perform many tests upfront, while we focus on the case of online optimization.

Our work on simultaneous learning and optimization of labor platforms has clear ties to Multi-Armed Bandit (MAB) problems \cite{bubeck2012regret} and indeed Johari et al. \cite{johari2016know} use MABs to match single workers (with quality uncertainty) to known job types (without uncertainty), though this work does not deal with the complex interdependency in the worker population when teams of workers are matched. Our work has closer ties to combinatorial bandits \cite{cesa2012combinatorial} and semi-bandits \cite{kveton2015tight}. In combinatorial bandits, there are finite possible basic actions each giving a stochastic reward and at each time one chooses a subset of these basic actions and receives the sum of the rewards for the subset chosen. To apply this framework to our problem, we can model each team as a basic action and at each time step we choose a subset that corresponds to a partition of the workers into teams. Existing regret bounds either assume independence between basic actions \cite{combes2015combinatorial}, which our setting violates, or allow for correlations \cite{kveton2015tight} but do not fully leverage the additional knowledge gained from how these basic actions are related, e.g. how the performance of two teams are related if they share team members. In this way, the teamwork setting we consider induces additional structure that can be leveraged to minimize regret. 

Another line of related work within the MAB literature is \cite{simchowitz2016best}, which considers the standard multi-armed bandit setting with $N$ arms, but here the decision at each step is to choose $K$ out of these $N$ arms, and only the highest reward out of these arms is revealed. The goal is to learn the best subset containing $K$ arms. Although there are similarities with our ``Strongest Link'' payoff structure, this setup models the problem of determining the best team for a single task at each time step and does not model a labor platform which must fill numerous jobs simultaneously.

\section{The Model}\label{sec:model}
We now describe our model formally.
\begin{enumerate}[leftmargin=*]
\item {\bf Workers and Types:} Suppose that we have a single job type and we have $N$ workers denoted by the set $\mathcal{N}\triangleq \{1, \ldots, N\}$. We assume that $N$ is even. Each job requires two workers to complete. Each worker has type $\theta_i \in \{0,1\}$. We assume that the $\theta_i$ are i.i.d. $\textup{Bernoulli}\,(p)$ random variables. For convenience, we denote $q=1-p$. The type of the worker represents the skill of the worker at the given job. We sometimes refer to workers of type $1$ as high quality workers and those of type $0$ as low quality workers. Let $N_1$ be the number of 1 workers in the population, which is distributed as $\textup{Binomial}\,(N,p)$. 
We assume that the platform knows $p$, but the specific type of each worker is unknown. 
We are interested in the regime where the number of workers is large, i.e., we allow $N$ to scale to infinity while the expected proportion of 1 workers in the worker pool, $p$, remains constant. 

\item {\bf Decisions and System Dynamics:} Matching decisions are made at times $t=0,1,\ldots$. The workers enter the system at time $t=0$ and the platform has no information about them except for the prior $p$. Workers stay in the system indefinitely. At each time, the platform creates a partition of the worker pool into pairs of size 2. Let $\Psi_{N}$ denote the set of all possible pairings. Let $A(t) = \{(i,j)\}\in \Psi_N$ denote the pairings of the workers at time $t$. Each team $(i,j)$ is assigned to a distinct job (jobs are assumed to be abundant) and receives a payoff at the end of the period, which is perfectly observed by the platform. 

\item {\bf Feedback model:} We consider two different feedback models.
In each of the two models, the score that a pair $(i,j)$ receives upon completion is a deterministic function of the worker types: 
\begin{enumerate}
\item Weakest Link: The score received by a pair $(i,j)$ is $P(\theta_i, \theta_j) = \min\{\theta_i, \theta_j\}$;  
\item Strongest Link: The score received by a pair $(i,j)$ is $P(\theta_i,\theta_j) = \max \{\theta_i, \theta_j\}$. 
\end{enumerate} 
Let $B(t) = \{ ((i, j), P(\theta_i, \theta_j) )\}$ denote the pairings and the payoffs observed for each pairing at time $t$.  

\item {\bf Policies:} In any period, the platform has access to the history of all previous pairings made and scores observed. 
A policy $\phi$ for the platform is a sequence of mappings, indexed by time, from the history $H(t) = \{ B(0), \ldots, B(t-1) \}$ (where the initial history is defined to be $H(0) = \emptyset$) to the next action $A(t)$. We denote the set of all policies for the setting with $N$ workers by $\Phi_N$. 
The total payoff generated for the platform at time step $t$ is $\sum_{(i,j) \in A(t)} P(\theta_i, \theta_j)$.



\item {\bf Objective:} 
Let $\bm{\theta}=(\theta_j\in\{0,1\}; j=1,\cdots, N)$ be a type assignment across workers and let $\Theta_N$ be the set of all such type assignments. The distribution of the type of each worker induces a distribution on the type assignments in $\Theta_N$. All the expectations in the remainder of the paper are defined with respect to this distribution.
The performance of a policy under type uncertainty can be compared to the performance of the optimal matching policy when the identities of the high quality workers are known.
For a fixed policy $\phi$ and a type assignment $\bm{\theta}$, we define the $T$ period per-worker regret as,

$$\textup{Regret}_N(\bm{\theta},\phi, T)=\frac{T}{N} \max_{A\in \Psi_N}\sum_{(i,j)\in A}P(\theta_i, \theta_j) - \frac{1}{N} \sum_{t=1}^T \sum_{(i,j) \in A(t)} P(\theta_i, \theta_j).$$
Note that this is a deterministic quantity given the type assignments $\bm{\theta}$. Since the workers are expected to stay in the platform indefinitely, the platform will eventually learn every worker's type and not incur any additional regret thereafter.  For instance, this can be achieved in $N-1$ stages by matching every worker to every other worker. To see this, simply enumerate all $N(N-1)/2$ pairs, observe that $N/2$ pairs can be covered in a single stage; thus everyone will be matched to everyone else in $N-1$ stages. 

Define $\tau$ to be the (random) time taken till no additional regret is incurred under a given policy $\phi$.
The platform then seeks to minimize the expected regret until $\tau$. Defining
\begin{align*}
\textup{Regret}_N(\phi)\triangleq \textup{E}[\textup{Regret}_N(\bm{\theta},\phi, \tau)],
\end{align*}
where the expectation is over the randomness in $\theta$, the platform seeks to solve the optimization problem:
$$\inf_{\phi \in \Phi}\textup{Regret}_N(\phi).$$
In fact, we will show that there exist good policies such that $E(\tau)=\textup{o}(N)$ time steps.

\end{enumerate}
\subsection{A sufficient statistic for the state: A collection of graphs}
In general, a matching policy could depend on the history of matches and rewards, but this state space quickly becomes intractable. We need a sufficient statistic that only preserves information relevant to the decision problem. At first glance we might believe that it is sufficient to simply preserve the marginal posterior distribution of each worker's type. Further consideration would reveal that this is not so and that we need the \emph{joint} posterior distribution of the types of all workers, which is a complex object. We nevertheless make the following observations that considerably simplify the analysis: 1) Once a worker's type has been identified, one does not need to preserve information about the worker's previous matchings and 2) The joint posterior distribution does not depend on the order in which matches were made. 

These two observations allow us to represent the state space as a collection of graphs. We describe this space in the Weakest Link model. Let us define the \emph{unknown worker graph} $G_t = \{V_t, E_t\}$ as follows:
\begin{itemize}
\item $V_t = \{ i\in\mathcal{N} \ \mid \ \Pr(\theta_i = 1| H(t)) \not\in \{0,1\} \ \}$
\item $(i, j) \in E_t$ if $(i,j) \in A(0) \cup \ldots \cup A(t-1)$ and $P(\theta_i,\theta_j) = 0$. 
\end{itemize}
The vertex set $V_t$ represents the set of unknown workers at time $t$ and an edge exists between two vertices if the two workers have previously been matched and received a reward of 0. A similar description can be done in the Strongest Link model also, except that an edge exists between two vertices if the two workers have previously been matched and received a reward of 1. 

In either setting, we can compress the history $H(t)$ into the statistic $(KH_t, KL_t, G_t)$:
\begin{equation}\label{eq:min-graph}
\begin{gathered}
KH_t = \{i\in\mathcal{N} \ | \  \Pr(\theta_i = 1|H(t)) = 1\} \\
KL_t = \{ i\in\mathcal{N} \ | \ \Pr(\theta_i=1|H(t)) = 0\} \\
G_t = \{V_t, E_t\}
\end{gathered}
\end{equation}
%
We can then show the following.
\begin{lemma}\label{lma:graph}
In the Weakest Link (and the Strongest Link) model, there is no loss in objective if we restrict ourselves to policies that depend only on $(KH_t,KL_t,G_t)$ in each period $t$.
\end{lemma}


%
%

\section{Weakest Link}\label{sec:min}
In this section, we describe our policies and results for the Weakest Link model. Before we do so, a brief discussion on the aspects of regret and learning is in order.

{\bf Regret in the Weakest Link Model.} In the Weakest Link model, recall that for any pair of workers $(i,j)$ with $\theta_i, \theta_j \in \{0,1\}$, the reward observed for this pair is $P(\theta_i,\theta_j) = \min\{\theta_i, \theta_j\}$ and thus a positive reward is generated only when $\theta_i = \theta_j = 1$. Thus, if the worker qualities are known, then the optimal matching strategy is to match the 1 workers amongst themselves. In the case where the number of 1 workers is even, every time a 1 worker is matched to a 0 worker, the platform incurs a (per-worker) regret of $1/2$. Thus minimizing regret amounts to minimizing the number of $0-1$ matches. When the number of 1 workers is odd, this is not precisely true, since one $0-1$ match is inevitable in each stage. But nevertheless, this is a good enough approximation to the objective when $N$ is large since we will be considering policies that incur no additional regret after in an expected $o(N)$ time steps. This is formalized in the following lemma.

\begin{lemma}\label{lma:weak-max-regret}
Consider a policy $\phi$ such that $E(\tau)\leq\textup{o}(N)$. Define $\textup{Regret}'_N(\bm{\theta},\phi, \tau)$ to be $1/(2N)$ times the total number of matches between high and low type workers under the policy until $\tau$. Let $\textup{Regret}'_N(\phi)\triangleq\textup{E}[\textup{Regret}'_N(\bm{\theta},\phi, \tau)]$. Then $|\textup{Regret}'_N(\phi)-\textup{Regret}_N(\phi)|\leq \textup{o}(1)$.
\end{lemma}

{\bf Learning in the Weakest Link Model.} Under the Weakest Link model, a 1 worker gets identified when he is matched to other 1 worker, either known or unknown, and a feedback of $1$ is observed. While 0 workers are identified when they get matched to another worker who is known to be a 1 or is later identified as being 1. Thus matching of a worker with another 1 worker is critical to her quality identification.

{\bf Inevitability of Regret.} Thus, informally, the goal of any optimal algorithm is to speedup the identification of all 1 workers, while minimizing the number of $0-1$ matches in the process. But the identification of unknown 1 workers requires pairing them with other 1 workers, and this in turn inevitably exposes unknown 0 workers to matches with 1 workers, thus incurring regret. 

The two central questions in this setting are: 1) How should the workers with unknown quality be matched amongst themselves? And 2) Should the identified high quality workers be matched amongst themselves or should they be matched to unknown workers? In order to address the two questions separately, we define the following class of policies.

\begin{definition}[A non-learning policy]
A policy is called \emph{non-learning} if it always pairs workers known to have type $\theta_j = 1$ with each other and not with workers whose types are still unknown to the system. 
\end{definition}
First note that feasible non-learning policies exist, because when unknown workers are matched with each other, 1 workers are always identified in pairs and so we do not have a case when known 1 workers are forced to be matched to unknown workers. To define such a non-learning policy, we need only to specify how workers whose qualities are unknown are matched amongst themselves. In other words, we can define the policy on the unknown graph, where vertices are unknown workers and edges indicate previous matchings with a reward of 0. 

This class of non-learning policies may seem attractive to the platform, as it guarantees that proven high quality workers will have positive future experiences and the platform may not wish to expose good workers to possible negative experiences. These policies are also myopic and maximize the number of certain high-quality pairs at each time step. We will show, however, that these policies are not always optimal from the perspective of the platform in the long run. 

We first show the following lower bound on the expected regret of any non-learning policy.

\begin{proposition}[A Lower bound for any non-learning policy]\label{thm:lbeven}
For any $N$ and $p$ and for any non-learning policy $\phi$ such that $E(\tau)\leq\textup{o}(N)$, we have
$$ \textup{Regret}_N(\phi)\geq \frac{3}{4}(1-p)-\textup{o}(1).$$
\end{proposition}

Next, we define two non-learning policies, a meta-policy Exponential Cliques and its practical implementation called $k$-stopped Exponential Cliques, and we will show the latter to be optimal among the class of non-learning policies for all $p$ when $N$ is large. Moreover, we will later show that it is asymptotically optimal among all policies, learning and non-learning, for our regret minimization problem as $p\rightarrow 1$. \\

\noindent\textbf{Exponential Cliques: }
As this is a non-learning policy, all known 1 workers are paired with each other and so we restrict our attention to the unknown graph. The algorithm is carried out in epochs and in each epoch we pair two cliques of unknown workers and test all pairwise matches between the two cliques. 
\begin{enumerate}[leftmargin=*]
\item At the first epoch $k=0$ (at time $t=0$), each group is an individual worker (vertex in $G_0$) and we pair workers at random and observe the feedback. Workers who are identified as 1 are removed from the unknown graph and the remaining vertices form a clique of size 2. 
\item In the remaining epochs $k=1, 2, \ldots$ the unknown graph consists of cliques of size $2^k$ at the start of the epoch. We pair cliques together and use the next $2^k$ time steps to make all pairwise matches between the workers in the two cliques. 
\item We stop when all worker types have been identified, or the graph consists of one clique.
\end{enumerate}
If at the start of an epoch $i$ there is an unpaired clique $C$, which has size $2^i$, then the workers in this clique will be matched among themselves until the next epoch $j$ when there is an unpaired clique $C'$, which has size $2^j$. In the $j$th epoch, we create all pairwise matches between $C$ and $C'$, which will take time $2^j$, and then repeat existing matches for the remainder of the epoch. \\

\begin{figure}[h]
\begin{minipage}[t]{.45\textwidth}
\includegraphics[height=1.2in]{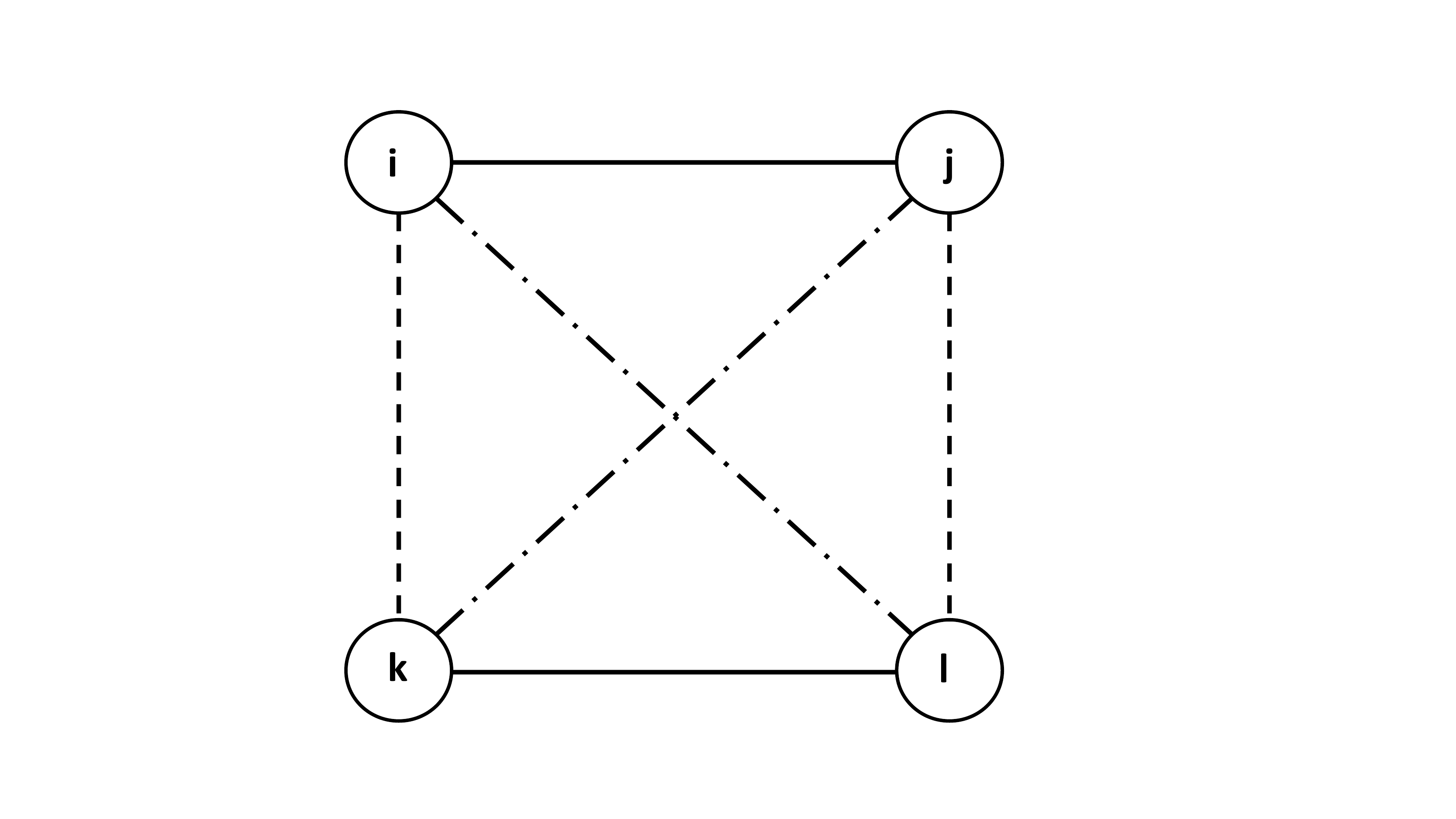}
\caption{{\small Exponential cliques at the beginning of the second epoch. Solid line represents matches that have resulted in an outcome of $0$, dotted line represents proposed next match, and dash-dotted line represents the match after if none of the dotted matches lead to an outcome of $1$.}}\label{fig:expclique1}
\end{minipage}\hfill
\begin{minipage}[t]{.45\textwidth}
\includegraphics[height=1.5in,angle=0]{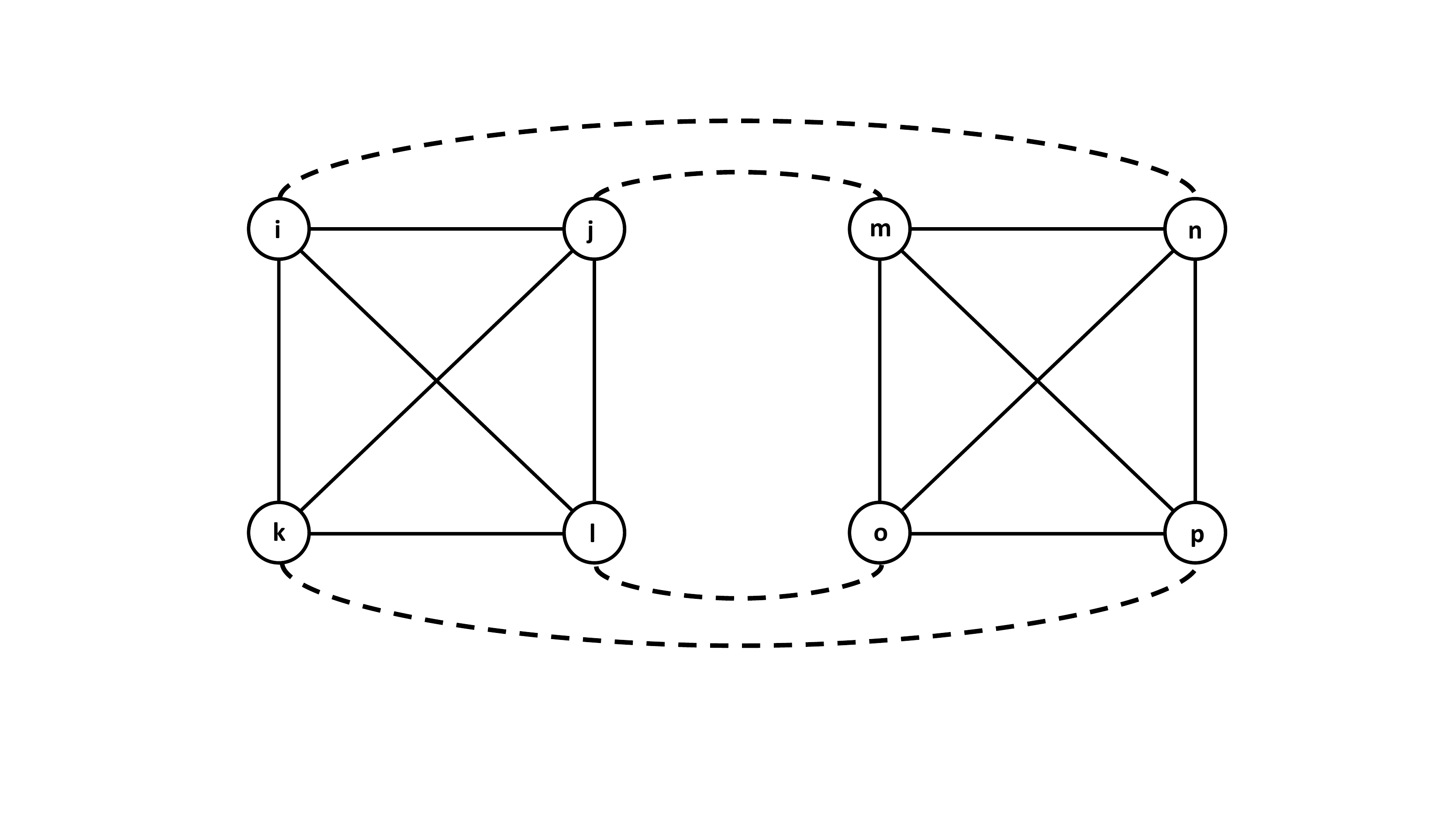}
\caption{{\small Exponential cliques at the beginning of the third epoch. Solid line represents matches that have resulted in an outcome of $0$, dotted line represents proposed next match.}}
\label{fig:expclique2}
\end{minipage}
\end{figure}

Note that in this algorithm, a worker's type is identified if and only if two cliques $C_1, C_2$, each with exactly one 1 worker, $w_1, w_2$, respectively, are paired together \emph{and} the pairing between $w_1$ and $w_2$ occurs. The fact that $\theta_{w_1}=1$ now tells us that all other workers in $C_1$ were 0. Similarly we can conclude that all workers in $C_2$ are also 0. If we do not observe a matching between such a $w_1$ and $w_2$, then at the end of the epoch $C_1$ and $C_2$ are combined to form a clique of twice the size. This algorithm maintains the invariant that at the beginning of every epoch $k$, all workers in the unknown worker graph are in a clique of size $2^k$, except for possibly workers in the unpaired clique. Since the workers in this unpaired clique are repeating matches among themselves, the platform is not learning any information about this clique and so these workers may be incurring more regret. As the number of epochs increases, the number of workers in this unpaired clique grows exponentially. To curtail the effects that this clique might create in terms of regret, we define the following algorithm that stops Exponential Cliques after $k$ epochs.

\noindent\textbf{k-stopped Exponential Cliques: } Run Exponential Cliques for $k$ epochs. After $k$ epochs, enumerate all matches remaining create remaining matches as fast as possible.

If there are $j$ unknown workers remaining when we stop Exponential Cliques, we can complete all remaining matches in at most $j$ time steps.

We have the following upper bound on the performance of this algorithm.
\begin{proposition}[Upper bound on expected regret of  k-Stopped Exponential Cliques]\label{thm:missing}
Let $\phi$ be the $\floor{\sqrt{2\log(N)}}$ stopped Exponential Cliques Algorithm. Then, 
$$ \textup{Regret}_N(\phi) \leq \frac{3}{4}(1-p) + \textup{o}(1)$$ 
and moreover, under the Exponential Cliques Algorithm, $E(\tau) =\textup{o}(N)$.
\end{proposition}

And thus $\floor{\sqrt{2\log(N)}}$-Stopped Exponential Cliques is an asymptotically optimal (when $N$ is large) non-learning algorithm.

Intuitively, this upper bound can be explained as follows. Under exponential cliques, at each epoch we have two types of cliques: Type A consisting of all 0 workers and Type B consisting of a single 1 worker, the remaining workers being 0. Each 0 worker starts in a singleton clique of type A. There are two steps to each 0 worker being identified. It first becomes a part of a type B clique, and in the process gets matched to a 1 worker exactly once (suppose for a moment that there are no unpaired cliques at any epoch, so the 0 worker doesn't get matched to the same 1 worker again). In the second step, this type B clique gets matched to another type B clique. In this step, there are two possibilities: either the 0 worker gets matched to the lone 1 worker of the other clique, or before that happens, the two lone 1 workers in the two cliques get matched to each other, thus identifying everyone in the two cliques. In the first case, the 0 worker gets matched to a 1 exactly once more, while in the latter case it doesn't. The two possibilities are equiprobable. Thus in expectation, each 0 worker gets matched to $3/2$ 1 workers, leading to a regret of $3/4$ per 0 worker. Thus the total regret per worker is upper bounded by $3/4(1-p)$. Of course, this argument assumes that there are no unpaired cliques at any epoch. The proof shows that by the careful selection of the stopping time for exponential cliques, the contribution to regret from the unpaired cliques until the stopping time and the regret from the residual matches after the stopping time are $\textup{o}(1)$.

To summarize, we have characterized an algorithm that is optimal among the class of algorithms that does not pair known high quality workers with workers whose types are uncertain. We have not yet made any claims as to how well this class of algorithms compares to policies that pair known workers with unknown workers for the sake of learning. 





\subsection{Learning vs. non-learning policies in the regime $ p \rightarrow 1$}

Now we show a lower bound on the expected regret of any policy as $p\rightarrow 1$, which will then allow us to conclude that the optimal non-learning $\floor{\sqrt{2\log(N)}}$-stopped Exponential Cliques policy is indeed asymptotically optimal (in $p$) among all algorithms, non-learning and otherwise, in this regime.

\subsubsection{A Lower Bound on Expected Regret}
The following lower bound holds for all policies and comes from the fact that all algorithms, in expectation, must incur a substantial amount of regret in the first two time steps since worker types are initially unknown. 

\begin{proposition}\label{thm:lblargep}
For any $p$ and policy $\phi$, we have 
$$\textup{Regret}_N(\phi)  \geq \frac{1}{2} \cdot \frac{(1-p)(p+2p^2)}{1+p} + \textup{o}(1) .$$
\end{proposition}

Since $\lim_{p\rightarrow 1}\big((1-p)(p+2p^2)/2(1+p)\big)\big/\big(3(1-p)/4\big)= 1,$ this shows that the non-learning $\floor{\sqrt{2\log(N)}}$-stopped Exponential Cliques is asymptotically optimal in the large $p$ regime and that risking bad matches by pairing known high quality workers with unknown workers is not necessary for the platform. 

\begin{comment}
For large values of $p$, with high probability, a large enough number of 1 workers will get identified in the first period so that we are able to learn all the remaining unknown workers in the second period. We can show that this policy is strictly suboptimal. Informally, the number of $0-1$ matches in the first period is $Np(1-p)+\textup{o}(N)$. Moreover the number of unidentified 0 workers in this period is $N(1-p)^2 +Np(1-p)+\textup{o}(N)$, all of which will be matched to a 1 worker in the second period and get identified. Since a regret of $1/2$ is incurred per worker in a $0-1$ match, the total per worker regret is $p(1-p)+(1-p)^2/2 +\textup{o}(1)$.  Now $\lim_{p\rightarrow 1}\big(p(1-p)+(1-p)^2/2\big)\big/\big(3(1-p)/4\big) = 4/3>1$.
\end{comment}

 




\subsection{ Suboptimality of non-learning algorithms for smaller $p$ }

We now show that non-learning algorithms are suboptimal for $p < 1/3$. We do so by proposing a learning algorithm that achieves a lower regret than $3/4(1-p) +\textup{o}(1)$ in this regime. This algorithm is a modification of the $k$-stopped Exponential Cliques algorithm and again ensures that the set of workers in the unknown graph are partitioned into cliques of equal size. This time, a known 1 worker is assigned to the clique to match with the unknown workers and learn their types faster. \\

\noindent\textbf{Distributed Learning: } This policy proceeds in epochs. 

At time $t=0$, we set epoch $k=0$ and match the workers at random. In the following epochs $k=1, 2, \ldots$ 
\begin{enumerate}[leftmargin=*]
\item Randomly pair cliques $(C_i, C_j)$ in the unknown worker graph. Form pairs $(a_i, a_j)$ of known 1 workers and assign each pair to one of the clique pairs $(C_i, C_j)$ until there are no more $(a_i, a_j)$ or $(C_i, C_j)$ pairs. We refer to the vertices in these $(C_i, C_j)$ pairs as the \emph{exploration set}.
\item If there are an odd number of cliques, choose one clique $C_i$ at random and pair workers in the clique with each other for this epoch.
\item For any $(C_1, C_2)$ not assigned to an $(a_1, a_2)$, run Exponential Cliques on this pair. 
\item For clique pairs $(C_1, C_2)$ in the \emph{exploration set} we use the known 1 workers $(a_1, a_2)$ to help learn $(C_1, C_2)$. At each step, create pairwise matchings between the two cliques but replace one matching $(w_1, w_2)$ with the matchings $(a_1, w_1)$ and $(a_2, w_2)$. It is possible to do this in such a way that no $w_1 \in C_1, w_2 \in C_2$ are matched to each other more than once.\footnote{To ensure that no two workers are matched more than once, at each step $r$ in epoch $i$, match $(a_1, C_1^{(r)} )$ and $(a_2, C_2^{(2^i - r \mod 2^i)})$, where $C_i^{(r)}$ denotes the $r$th worker in $C_i$. Then for all other $s \neq r$, pair $(C_1^{(s)}, C_2^{(s-r \mod 2^i)} )$.  }

\item If the algorithm identifies a 1 worker $w_1 \in C_1$, pair $w_1$ with any known 1 workers available and pair the other 0 workers in $C_1$ with 0 workers. Use the known 1 worker $a_2$ to learn the rest of worker types in $C_2$. 
\end{enumerate}
This policy maintains the invariant that all vertices in the unknown graph are contained in cliques of size $2^k$ at the beginning of each epoch, except for possibly one unpaired clique. Since the algorithm identifies the worker types of all cliques in the exploration set by the end of an epoch, the only cliques remaining will be those which underwent the original Exponential Cliques policy without learning. Thus at the start of the next epoch all workers will be contained in a clique of size $2^{k+1}$.
\begin{figure}[h]
\begin{minipage}[t]{.45\textwidth}
\includegraphics[height=1.2in]{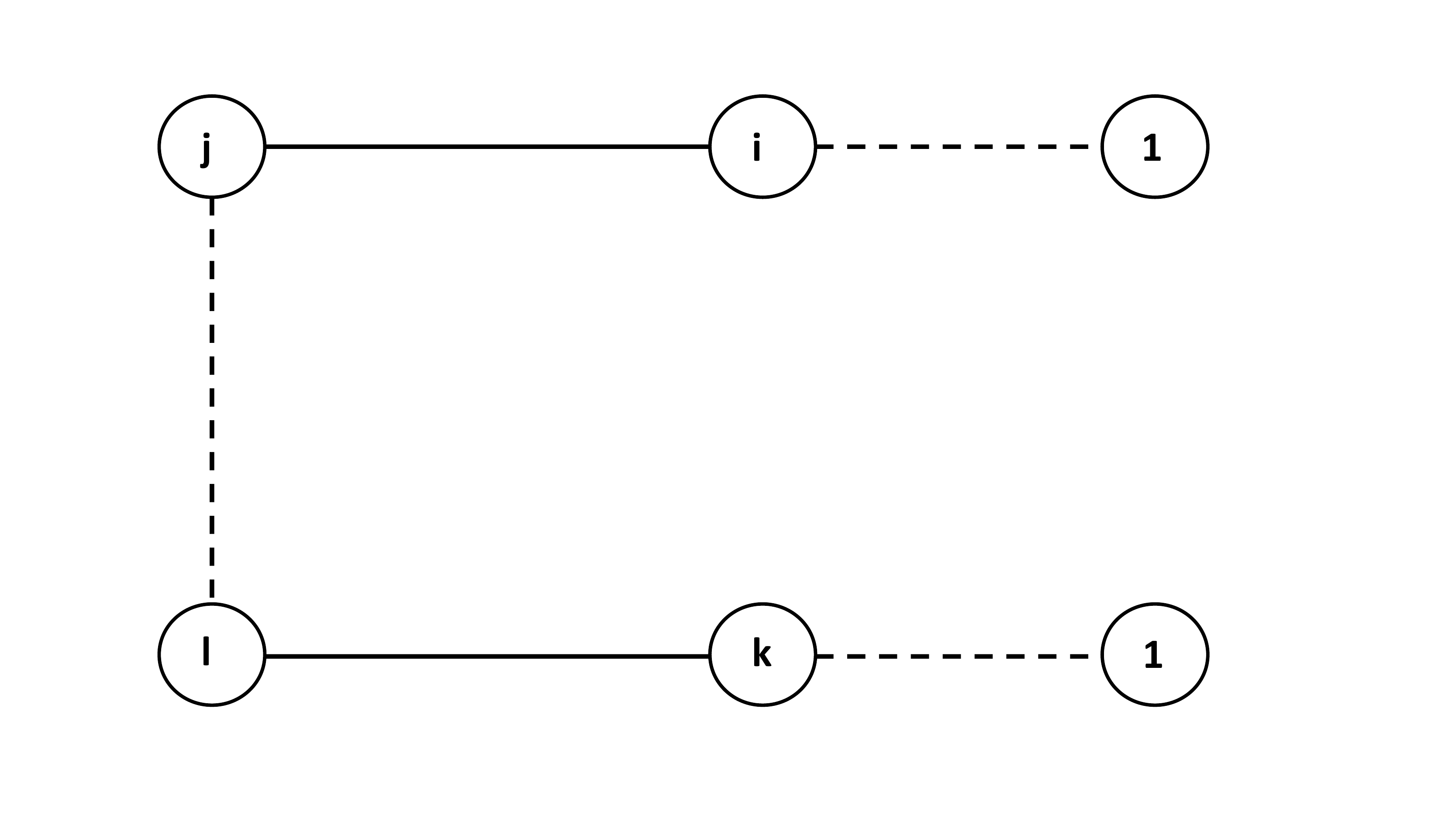}
\caption{\small{Distributed learning at the beginning of the second epoch. Dashed lines represent past matches that led to an outcome of $0$. Dotted lines are the proposed matches.}}\label{fig:distlearn1}
\end{minipage}\hfill
\begin{minipage}[t]{.45\textwidth}
\includegraphics[height=1.5in,angle=0]{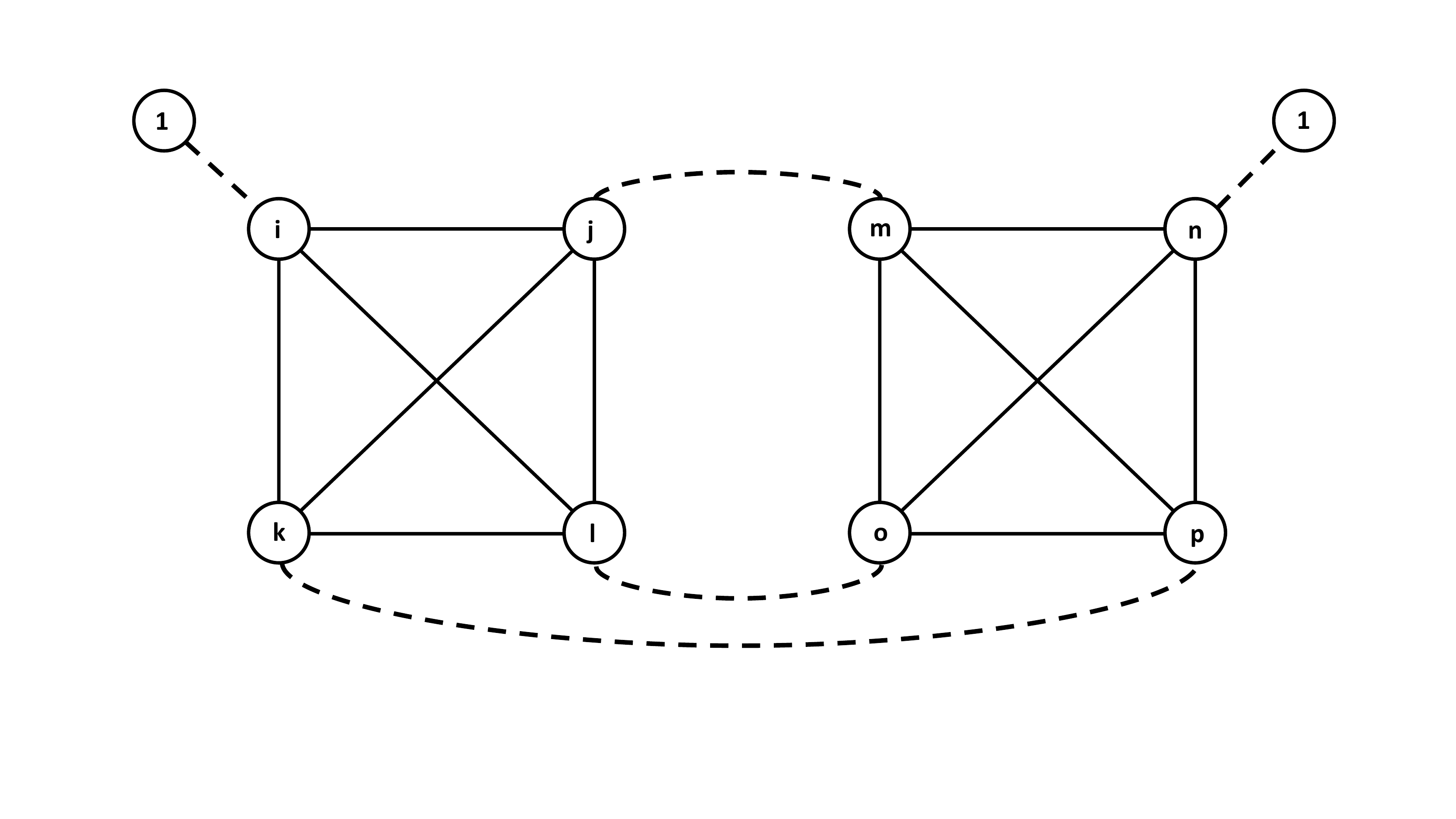}
\caption{{\small Distributed learning at the beginning of the third epoch. Dashed lines represent past matches that led to an outcome of $0$. Dotted lines are the proposed matches.}}
\label{fig:distlearn2}
\end{minipage}
\end{figure}
The fact that the inputs and outputs of Distributed Learning at each epoch are cliques of size $2^k$ and $2^{k+1}$, respectively, allows us to alternate epochs of Exponential Cliques and Distributed Learning. 

\textbf{$k$-stopped Exponential Cliques with $j$-learning ($\textup{EC-L}_j$)}
Let $\textup{EC-L}_j$ denote the policy which matches at random at $t = 1$, conducts $j$ rounds of Distributed Learning, then runs $k$ epochs of Exponential Cliques, and then completes all remaining matches as fast as possible.


The following proposition shows that this learning policy can beat the performance of all non-learning policies for small enough $p$. 

\begin{proposition}\label{thm:ubec-l1}
For $p< 1/3$, $\lim_{N\rightarrow\infty}\textup{Regret}_N(\textup{EC-L}_1) < \frac{3}{4} (1-p).$
\end{proposition}

The Weakest Link model exhibits a tradeoff between myopic regret minimization and learning, in terms of who to pair known high type workers to. We showed that our non-learning policy is optimal as $p \rightarrow 1$ and that all non-learning policies are suboptimal for $ p \in [0, 1/3]$. Thus if the platform wishes to maximize the expected aggregate payoff across all workers, a greedy policy is not always optimal and the platform must take into account the proportion of high quality and low quality workers in its worker pool.

\section{Strongest Link}\label{sec:max}
We now consider the strongest link setting. In particular, for any pair of workers $i,\,j$, the payoff that the team formed by this pair generates is $P(i,j) = \max\{\theta_i, \theta_j\}$.




\textbf{Regret in the Strongest Link model.}
In this setting, recall that for any pair of workers $(i,j)$ with $\theta_i, \theta_j \in \{0,1\}$, the reward observed for this pair is $P(\theta_i,\theta_j) = \max\{\theta_i, \theta_j\}$, and a positive feedback is generated when $\theta_i = 1$ or $\theta_j = 1$. Thus, if the worker qualities are known, the optimal matching is to make $0-1$ matches to the extent possible. When $p<0.5$, and when $N$ is large, some $0-0$ matches are inevitable in the optimal matching. Thus regret can be measured by the number of $1-1$ matches. On the other hand if $p>0.5$, some matches $1-1$ are inevitable and regret can be measured by the number of $0-0$ matches. We formalize this observation in the following lemma. The proof can be found in the appendix.
\begin{lemma}[Characterization of Regret]\label{lma:char of max regret}
For any policy $\phi$ such that $E(\tau)\leq\textup{O}(1)$, define $\textup{Regret}''_N(\bm{\theta},\phi, \tau)$ to be $1/N$ times the number of $0-0$ $(1-1)$ matches for $p > 0.5$ ($p < 0.5$) until $\tau$. Let $\textup{Regret}''_N(\phi)\triangleq\textup{E}[\textup{Regret}''_N(\bm{\theta},\phi, \tau)]$. Then $|\textup{Regret}''_N(\phi)-\textup{Regret}_N(\phi)|\leq \textup{o}(1)$.
\end{lemma}

\textbf{Learning in the Strongest Link model.}
In this setting, 0 workers are identified by being matched other 0 workers, either known or unknown. While 1 workers are identified when they are matched to either a known 0 worker, or an unknown 0 worker that later gets identified. Thus, unlike the Weakest Link model, a pairing of an unknown worker to a \emph{0} worker is necessary for identification.

{\bf Inevitability of Regret.}  In any algorithm, in the first period, workers are paired off arbitrarily. At the end of this period, there is the set of pairs of workers $K$ that have been identified as 0 (ones that resulted in an outcome of $0$) and the set $U$ of pairs of workers of unknown quality (ones that resulted in an outcome of $1$). Of the pairs that are of unknown quality, it is straightforward to compute that $r = p/(2-p)$ fraction are of the $1-1$ type and the remainder are of the $0-1$ type. Thus in the first period itself we have an unavoidable fraction of matches that incur regret irrespective of $p$. Moreover, more regret is unavoidable in the subsequent periods. 

1) When $p>0.5$ the pairs of known $0-0$ workers in $K$ incur regret if they are matched to each other and the goal is to find these workers high quality matches from the unknown population as quickly as possible. But matching known 0 workers to unknown workers to learn inevitably incurs regret due to potential matches to unknown 0 workers. 

2) When $p<0.5$ the pairs of workers in $U$ that are of the $1-1$ type incur regret and they need to be matched to a 0 worker as soon as possible. Since unknown $0-1$ pairs already incur no regret, it is the known 0 workers at the end of the first period that must be matched to the unknown $1-1$ pairs in $U$. But in this case, matching these known 0 workers to unknown workers does not incur any regret. Regret is nevertheless inevitable because of the unavoidable $1-1$ matches between unknown workers.

In either case, the central question is: how to utilize workers that have been identified as 0 at the end of the first period to explore the unknown workers?


\subsection{Candidate algorithms}
After observing outcomes of the pairings in the first period, consider the situation from the perspective of a conservative manager. Since all the pairs of workers in $U$ result in a reward of $1$, she may be tempted to not jeopardize that guaranteed reward and hence match the learned 0 workers amongst themselves. But any manager can be convinced that this approach is catastrophic in terms of regret. The problem is that there could be $1-1$ type pairs in $U$, which could be separated into two $0-1$ pairs, thus doubling the reward. Thus any algorithm that eventually attains $0$ regret must pair the discovered 0 workers that do not have a high type match, to unknown workers. 

What, then, is the policy that matches these known 0 workers to unknown workers in the most conservative fashion, i.e., with the least negative impact on immediate payoffs? This policy is simply the one where every pair of known 0 workers are matched to an unknown pair of workers to the extent possible (remaining matches are preserved). By doing so, there is no adverse impact on immediate payoffs: either the unknown pair was $1-1$ in which case the payoff gets doubled, or it is $0-1$ in which case the immediate payoff (from that pair) remains the same, i.e., $1$. Every known $0-0$ pair is successively paired to an unknown pair until it gets matched to a $1-1$ pair. This policy thus eventually attains no regret. We call this policy $1$-chain policy, described formally below, for reasons that will become clear.

At this point, we would like to understand what is the power of this most conservative asymptotically regret optimal algorithm. In particular, are there gains from adopting a more aggressive experimentation approach? To understand this, we propose a slightly different algorithm that we refer to as a $2$-chain algorithm, formally described below. We show a remarkable reversal of the relative performance of the $1$-chain and $2$-chain algorithms at $p=0.5$: the former incurs lower regret for $p<0.5$ where as the situation is reversed for $p>0.5$.

\begin{figure}[h]
\begin{minipage}[t]{.45\textwidth}
\includegraphics[height=1.2in]{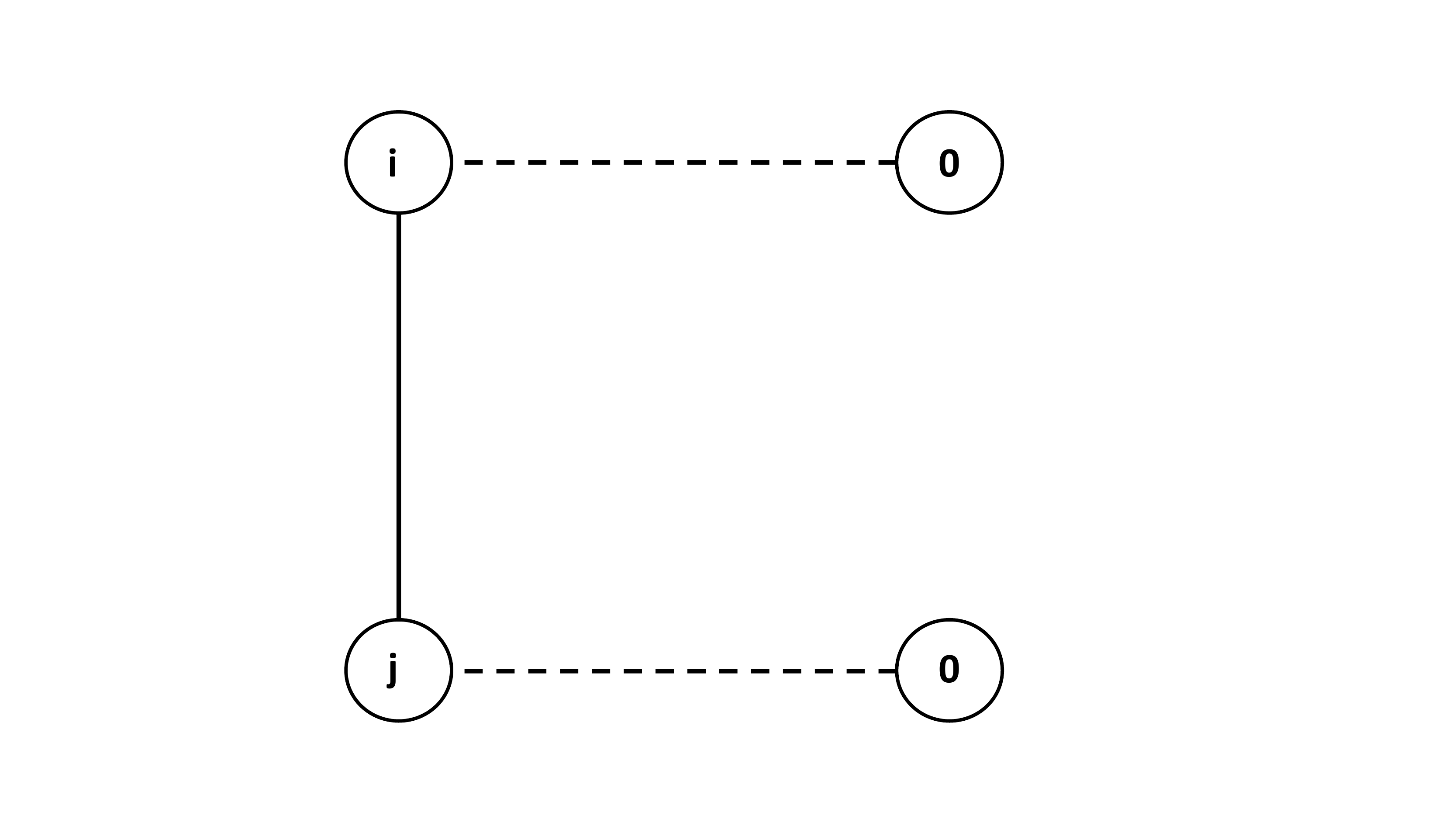}
\caption{{\small A $1$-chain epoch. Dashed lines represent past matches that led to an outcome of $0$. Dotted lines are the proposed matches.}}\label{fig:onechain}
\end{minipage}\hfill
\begin{minipage}[t]{.45\textwidth}
\includegraphics[height=1.2in,angle=0]{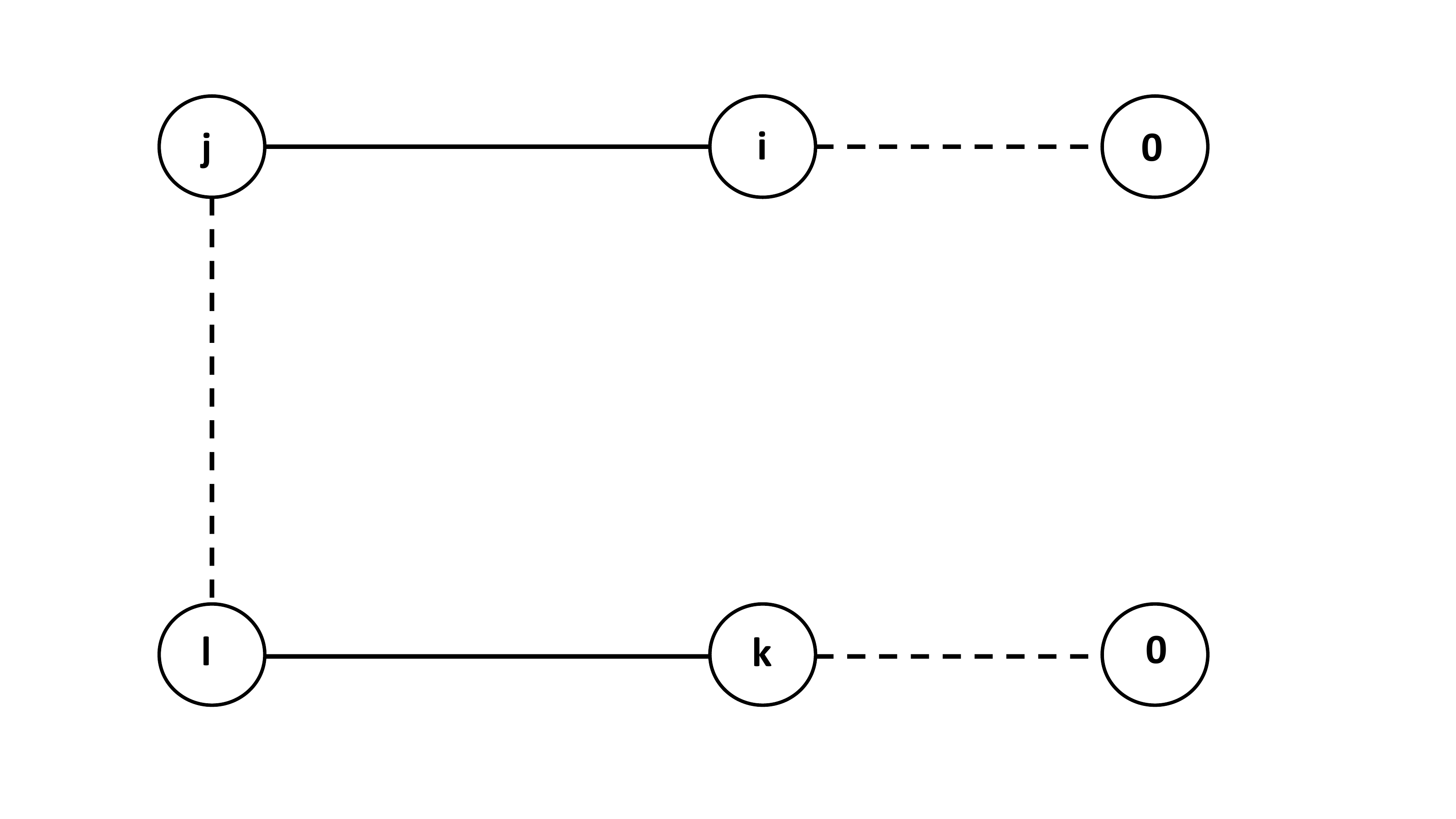}
\caption{{\small A $2$-chain epoch. Dashed lines represent past matches that led to an outcome of $0$. Dotted lines are the proposed matches.}}
\label{fig:twochain}
\end{minipage}
\end{figure}

\subsubsection{1-chain and 2-chain asynchronous algorithms}

We now define the $1$-chain and $2$-chain algorithms that we referred to above. At each time, the algorithms maintain 1) a dynamic set $U$ of pairs of workers of unknown quality and 2) a dynamic set $P$ of the set of pairs of workers of known quality that are matched to each other, and will be matched to each other for posterity. 

Now for each pair $0-0$ in $K$ (the set of pairs of workers that have been identified as being of low quality in period 1), the algorithms asynchronously operate in epochs to find high quality matches for each of the two low quality workers. The two algorithms differ in how the epochs are implemented.

\begin{itemize}[leftmargin=*]
\item {\bf $1$-chain.} While matches have not been found for the $0-0$ pair, pick a pair from $U$. Let it be $i-j$. Then match $0-i$ and $0-j$. If both matches result in an outcome of $1$, then add $0-i$ and $0-j$ to P and stop.  The $0-0$ pair has been matched. If not, then add $i-j$ to P and go to the next epoch. If $U$ is empty at the beginning of any epoch, add $0-0$ to $P$ and stop. Note that each epoch is of time $1$.
\item {\bf $2$-chain.} While matches have not been found for the $0-0$ pair, pick two pairs from $U$. Let they be $i-j$ and $k-l$. Then match $0-i$, $0-k$ and $j-l$. Now there are 4 cases:
\begin{enumerate}[leftmargin=*]
\item All matches result in an outcome of $1$; in which case add $0-i$, $0-k$ and $j-l$ to $P$ and stop.
\item One of $0-i$ or $0-k$ results in an outcome of $1$, the other resulting in $0$, and $j-l$ results in an outcome of $1$. In this case, let $0-i$ be the match that resulted in outcome 0. Then in the next time period, match $0-k$, $0-l$, and $i-j$. If $0-l$ leads to an outcome of $1$, then add $0-k$, $0-l$, and $i-j$ to $P$ and stop. Else, add $i-j$ and $k-l$ to $P$ and go to the next epoch.
\item Both $0-i$ or $0-k$ result in an outcome of $0$, and $j-l$ results in an outcome of $1$. In this case, add $i-j$ and $k-l$ to $P$ and go to the next epoch.
\item Both $0-i$ or $0-k$ result in an outcome of $1$, and $j-l$ results in an outcome of $0$. In this case, add $i-j$ and $k-l$ to $P$ and go to the next epoch.
\end{enumerate}
These are the only $4$ possibilities. If $U$ has one pair only at any epoch, implement the $1$-chain epoch. If $U$ is empty, add $0-0$ to $P$ and stop. Note that each epoch in this case is either of $1$ period  (cases 1, 3 or 4) or $2$ periods (case 2).
 \end{itemize}
 
At any time, any pairs in the set $U$ that are not demanded by any of the $0-0$ pairs in $K$ for experimentation, are matched to each other.
The following proposition summarizes our main finding.
\begin{proposition}
\begin{enumerate} \label{thm:2chain}
\item  \label{thm:2chain_a} When $p>0.5$, the $2$-chain algorithm incurs a lower regret compared to the $1$-chain algorithm.
\item \label{thm:2chain_b} When $p<0.5$, the $1$-chain algorithm incurs a lower regret compared to the $2$-chain algorithm. 
\end{enumerate}
\end{proposition}

In the proof, we show that the $1$-chain algorithm accumulates fewer matches between two high type workers for all $p \in [0,1]$. However, the $2$-chain algorithm results in fewer matches between two low type workers than the $1$-chain algorithm for $p> 0.5$. Thus the change in relative performance between $1$-chain and $2$-chain algorithms is a result of the discontinuity in the characterization of regret and not a discontinuity in the performance of the algorithms themselves.
\section{Conclusion}\label{sec:conclusion}
This work suggests that online labor platforms which facilitate flash teams and on-demand tasks can drastically improve the overall performance of these teams by altering the algorithms that design these teams. A common theme that we explored is what is to be done with workers whose qualities have been learned: should they be utilized to exploit or to explore? We demonstrated that the answer intricately depends on the payoff structure and and on the apriori knowledge about the distribution of skill levels in the population. We provided several insights into to these tradeoffs in the case of two natural payoff structures. In particular, through fundamental regret bounds, we showed that simple myopic strategies can be highly suboptimal in certain regimes, while can work well in certain others. We are optimistic that these insights can help guide effective managerial decisions on these platforms. 



\bibliographystyle{alpha}
\bibliography{sample-bibliography}




\appendix
\section{Proofs of all results}
\begin{proof}[Proof of Lemma~\ref{lma:graph}]

We only show the result for the Weakest Link model; the Strongest Link setting is similar. Note that the entire history can be represented in an edge-labeled graph $G_t' = (V', E_t')$ where $V = \{ u \}$ and $(u,v) \in E_t$ if $(u,v) \in A(0) \cup \ldots \cup A(t-1)$ and with labeling function $f: E' \rightarrow \{0,1\}$ such that $f: (u,v) \mapsto P(u,v)$. Note that the only pieces of information in the history $H(t)$ not captured in $G_t'$ are the (possible) repeated matchings between workers $u, v$ and the order in which these matches occurred. Neither affects any posterior probabilities of the worker or subsets of the workers. Thus this graph captures all joint posterior probabilities $\Pr(P(u, u') = 1)$ for all workers $u, u'$. 

Now suppose that we introduce a vertex-labeling function $g: V'' \rightarrow \{0,1\}$ on some subset $V'' \subseteq V'$ of the vertex set, to be defined, with $g: v \mapsto \theta_v$. 

For any edge $(u,v) \in E'$ with $f((u,v)) = 1$, it must be true that $\Pr(\theta_u = 1 | H_t) = \Pr(\theta_v = 1 | H_t) = 1$ and for any vertex $w \in V'$ with $f(u,w)=0$ or $f(v,w) = 0$ that $\Pr(\theta_w = 1 | H_t) = 0$. Then removing all edges adjacent to $u$ and $w$, adding $u,v$ and all neighbors $w$ to the set $V''$, and labeling the vertices with the worker types does not change the posterior probabilities that can be calculated from this graph. For any vertex $x \in V'$ adjacent to a vertex $w \in V'$ with $\Pr(\theta_w = 1 | H_t) = 0$, the edge label $f((x,w)) = 0$ regardless of $\theta_x$ and so removing this edge does not affect posterior probabilities. Thus \eqref{eq:min-graph} is sufficient to represent the history $H_t$. 
The proof for the Strongest Link setting is similar and follows from the same observations that repeated matchings and the order in which matchings occur do not affect the posterior distribution.
\end{proof}

\begin{proof}[Proof of Lemma~\ref{lma:weak-max-regret}.] 
When $N_1$ is even then every time a $0-1$ match is made, some other $0-1$ match is being made as well, where as the two high types should have been matched to each other. Thus there is a a loss in payoff of $1$ for every two $0-1$ pairs, i.e., a regret of $1/2$ per $0-1$ pair, or a regret of $1/(2N)$ per worker, per $0-1$ pair. Thus the result follows. When $N_1$ is odd, one $0-1$ match is inevitable at each time step and does not incur any loss in payoff. This attributing a regret of $1/2$ per $0-1$ pair over-estimates the regret by $\tau/(2N)$. But since $E(\tau)\leq \textup{o}(N)$, the result follows.
\end{proof}

\begin{proof}[Proof of Proposition~\ref{thm:lbeven}]
The proof follows multiple steps. First, observe that the random ability assignment where first the number of high types is drawn from a Binomial(N,p) distribution, and then the realized number is assigned uniformly across all possible assignments to the workers, has the same distribution as the i.i.d. ability assignment with probability p. In order to prove the result, first observe that 
\begin{align}
\textup{Regret}(\phi)\geq \textup{Regret}'(\phi)-\textup{o}(1)&=P(N_1\geq 2)\textup{E}[\textup{Regret}'(\bm{\theta},\phi, \tau)|N_1\geq 2] -\textup{o}(1)\\
&~~+P(N_1< 2)\textup{E}[\textup{Regret}'(\bm{\theta},\phi, \tau)|N_1< 2] -\textup{o}(1)\\
&=P(N_1\geq 2)\textup{E}[\textup{Regret}'(\bm{\theta},\phi, \tau)|N_1\geq 2] +0-\textup{o}(1)\\
&=(1-\textup{o}(1))\textup{E}[\textup{Regret}'(\bm{\theta},\phi, \tau)|N_1\geq 2]-\textup{o}(1).
\end{align}
We will thus prove a lower bound on $\liminf_{T \rightarrow \infty}\textup{E}[\textup{Regret}'(\bm{\theta},\phi, T)|N_1\geq 2]$. To do so, we prove a lower bound on $\liminf_{T \rightarrow \infty}\textup{E}[\textup{Regret}'(\bm{\theta},\phi, T)|N_1, N_1\geq 2]$, and then take conditional expectation on the event that $N_1\geq 2$ to derive the overall lower bound.  Fix any one of the $N-N_1$ low types. Call this type $l$. We will first show that the probability that the type $l$ gets matched to at least one high type is at least $(N_1-1)/N_1$.  

To see this, consider an assignment of types in which the low type $l$ does not get matched to any high type. Now for each such assignment, there are at least $N_1-1$ other assignments, obtained by swapping $l$ with each one of the $N_1-1$ high types that get identified, such that in every swapped assignment, $l$ gets matched to a high type (if $N_1$ is even, each high type will eventually get identified; if $N_1$ is odd, at least $N_1-1$ high types get identified). Further, no two assignments in which $l$ doesn't get matched to any high types lead to the same assignment after a swap. To see this, suppose that $i$, $j$ and $k$ are three different workers. Suppose that $h$ and $h'$ are two other high types. Suppose that assignment $1$ of these types is $i-h$, $j-l$ and $k-h'$, and assignment $2$ is $i-h'$, $j-h$ and $k-l$ (the assignment of types to all other workers is the same in the two assignments). Then swapping the assignments of $l$ and $h$ in assignment $1$ leads to the same assignment as that obtained by swapping assignments of $h'$ and $l$ in assignment $2$. Suppose that in assignment $1$, $l$ never gets matched to a high type. But since $h'$ gets matched to a high type in assignment $1$ (since $N_1$ is even), this means that in assignment $2$, $l$ will definitely get matched to a high type because the policy is unable to differentiate the two assignments until that happens (note that in assignment $2$, $h$ will not be matched to a high type till the policy notices the difference). Thus no two assignments in which $l$ doesn't get matched to any high types lead to the same assignment after a swap, because if that is the case, then in one of the assignments $l$ does get matched to a high type, which is a contradiction.

Thus if we denote the probability of $l$ being never matched to a high type as $x$, then we have $x+(N_1-1)x\leq 1$, which implies that $(1-x)\geq (N_1-1)/N_1$.

Next, fix a type assignment in which $l$ gets matched to exactly one high type that eventually gets identified. Let $h$ denote this high type. Thus in the assignment under consideration, after $h$ and $l$ are matched, $h$ gets matched to another high type. For every such type assignment there is another assignment where the assignments of $h$ and $l$ are swapped, in which after $l$ and $h$ get matched, $l$ gets matched to another high type before $h$ gets matched to another high type. The policy cannot tell the difference between these two assignments until this happens since the sequence of outcomes remains the same.
Thus for every assignment where $l$ gets matched to exactly one high type before being identified, there is an assignment where $l$ gets matched to at least 2 high types before getting identified. Moreover no two assignments lead to the same assignment after the swap. To see this, suppose that $i$, $j$ and $k$ are three different workers. Suppose that $h$ and $h'$ are two other high types. Suppose that assignment $1$ of these types is $i-h$, $j-l$ and $k-h'$, and assignment $2$ is $i-h'$, $j-h$ and $k-l$ (the assignment of types to all other workers is the same in the two assignments). Then swapping the assignments of $l$ and $h$ in assignment $1$ leads to the same assignment as that obtained by swapping assignments of $h'$ and $l$ in assignment $2$. Suppose that in assignment $1$, $h$ is the only high type that $l$ gets matched to. Then in assignment $2$, $h'$ and $h$ get matched before $l$ can get matched to $h'$ and are removed in the non-learning policy. Thus even if $l$ gets matched to exactly $1$ high type in assignment $2$, that high type is not $h'$ and hence the swap is not valid. Thus no two assignments in which $l$ gets matched to a single high type lead to the same assignment after the swap.

Now define the following quantities:
\begin{itemize}
\item Let $y_{N_1}$ be the probability that the type $l$ is matched to exactly one high type that eventually gets identified under the non-learning policy, 
\item  $z_{N_1}$ is the probability that the type $l$ is matched to two or more high types under the non-learning policy, 
\item $g_{N_1}$ is the probability that the type $l$ gets matched to exactly one high type, that never gets identified (this can happen only when $N_1$ is odd).
\end{itemize}
Then our argument above shows that $z_{N_1}\geq y_{N_1}$. Further $y_{N_1} + z_{N_1} + g_{N_1} = 1-x \geq (N_1-1)/(N_1)$. Thus $z  \geq (N_1-1)/N_1 - y_{N_1}-g_{N_1} \geq (N_1-1)/N_1 - z_{N_1}-g_{N_1}$. This means that $z_{N_1}\geq (N_1-1)/2N_1-g_{N_1}/2$.

Thus the expected number of high types that $l$ gets matched to under the policy is at least $y_{N_1} + g_{N_1}+ 2z_{N_1} \geq (N_1-1)/N_1 + z_{N_1} \geq 3(N_1-1)/2N_1- g_{N_1}/2$.

Thus $1/(2N)$ times the expected number of high-low matches across all types conditional on $N_1$ is, 
\begin{align*}
\frac{1}{2}\bigg(1-\frac{N_1}{N}\bigg)\bigg(\frac{3(N_1-1)}{2N_1}-\frac{g_{N_1}}{2}\bigg)&= \frac{1}{2}\bigg(1-\frac{N_1}{N}\bigg)\bigg(\frac{3}{2}(1-\frac{1}{N_1})-\frac{g_{N_1}}{2}\bigg)\\
&\geq \frac{1}{2}\bigg(1-\frac{N_1}{N}\bigg)\frac{3}{2} -\frac{3}{2N_1}-\frac{g_{N_1}}{2}.
\end{align*}
Thus, 
\begin{align*}
&\textup{E}[\textup{Regret}_N(\bm{\theta},\phi, \tau)|N_1\geq 2]\\
&~~\geq \textup{E}\bigg[\bigg(1-\frac{N_1}{N}\bigg)\frac{3}{4} -\frac{3}{4N_1}-\frac{g_{N_1}}{4}\,\bigg|\, N_1\geq 2\bigg] \\
&~~= \frac{3}{4}(1-p) -\textup{E}\bigg[\frac{3}{4N_1}\,\bigg|\, N_1\geq 2\bigg] - \textup{E}\bigg[\frac{g_{N_1}}{4}\,\bigg|\, N_1\geq 2\bigg] \\
&~~= \frac{3}{4}(1-p) -\textup{o}(1)- \textup{E}\bigg[\frac{g_{N_1}}{4}\,\bigg|\, N_1\geq 2\bigg] .
\end{align*}
We finally show that $\textup{E}[g_{N_1}\,\mid\, N_1\geq 2]=\textup{o}(1)$, which completes the proof. This quantity is the probability that a type $l$ gets matched to exactly one high type and that high type is the one that never gets identified. Suppose that this probability is $\textup{O}(1)$. Then the expected number of such matches till the time the algorithm finishes learning is $N(1-p)\textup{O}(1)=\textup{O}(N)$, which contradicts the fact that the algorithm learns in $\textup{o}(N)$ time steps in expectation (since a high type that doesn't get identified gets matched to a low type exactly once in each time step).
\end{proof}

\begin{proof}[Proof of Proposition~\ref{thm:missing}]
We need the following two lemmas to establish the proof.
\begin{lemma}\label{lemma:ProbCliqueHigh}
At the beginning of any epoch of the Exponential Cliques algorithm, the probability that a clique of size $m$ contains a high type worker is 
\begin{equation*}
\frac{m p}{m p + 1-p}.
\end{equation*}
\end{lemma}

\begin{proof}
First note that the initial worker types are independent and distributed according to $\textup{Bernoulli}(p)$. Then a subset of $j$ workers has $X$ high type workers, where $X \sim \textup{Binomial}(j, p)$. If a clique is in the unknown graph, then there can be at most one high type worker in the clique. Then the probability that the clique contains exactly one high type worker, given the fact that it contains either one or zero, is $\frac{m (1-p)^{m-1}p}{(m (1-p)^{m-1}pm  + (1-p)^m} = \frac{mp}{mp + 1-p}$. 
\end{proof}

\begin{lemma}\label{lemma:VarianceBound}
Given any $N \in \mathcal{N}$ and any sequence $\lambda_1, \lambda_2, \ldots$ such that $\lambda_i \in [0,1]$ for all $i \in \{1,2,\ldots \}$, define the random variables $\beta_t$ as follows:
\begin{align*}
\beta_{t} &\sim Bin(\lfloor \frac{\beta_{t-1}}{2} \rfloor, \lambda_t), \forall t \in \{1,2,\ldots\}, \\
\beta_0 &= N.
\end{align*}
Then for all $t  \in \{1,2,\ldots \}$,
\begin{align*}
\textup{E}[\beta_t] &\leq \frac{N}{2^t}\Lambda_t, \\
\textup{E}[\beta_t^2] &\leq  \left( \frac{N}{2^t}\Lambda_t \right)^2 + \frac{N}{2^t}\Lambda_t(1 - \Lambda_t),
\end{align*}
where $\Lambda_t=\prod_{j=1}^t\lambda_j$.
\end{lemma}
\begin{proof}
We will prove the above lemma by induction. 
Note that for $t=1$, the proof is straightforward as $\beta_1$ has the distribution $Bin(N/2,\lambda_1)$.

Also, for any $t$, the fact that $\textup{E}[\beta_t] \leq \frac{N}{2^t}\Lambda_t$ is fairly straightforward. For example, it follows from a coupling argument with a random variable of type $Bin(\lfloor \frac{N}{2^t} \rfloor, \Lambda_t)$.

Assume the lemma is true for $1 \leq t \leq k$. Then,
\begin{align*}
\textup{E}[\beta_{k+1}^2] &= \textup{E}\left[ \textup{E}[\beta_{k+1}^2 | \beta_k ] \right] \\
&= \textup{E}\left[ \lambda_{k+1}^2 {\lfloor \frac{\beta_k}{2} \rfloor}^2 + \lambda_{k+1}(1 - \lambda_{k+1})\lfloor \frac{\beta_k}{2} \rfloor \right] \\
&\leq \frac{\lambda_{k+1}^2}{4} \textup{E}[\beta_k^2] + \frac{1}{2}\lambda_{k+1}(1 - \lambda_{k+1}) \textup{E}[\beta_k].
\end{align*}

By the induction assumption, we have $\textup{E}[\beta_k^2] \leq \left( \frac{N}{2^k}\Lambda_k \right)^2 + \frac{N}{2^k}\Lambda_k(1 - \Lambda_k)$.  

Also remember that $\Lambda_t \lambda_{t+1} = \Lambda_{t+1}$.

Therefore, we get that
\begin{align*}
\textup{E}[\beta_{k+1}^2] &\leq \frac{\lambda_{k+1}^2}{4} \left( \left( \frac{N}{2^k}\Lambda_k \right)^2 + \frac{N}{2^k}\Lambda_k(1 - \Lambda_k) \right) + \lambda_{k+1}(1 - \lambda_{k+1}) \Lambda_k \frac{N}{2^k} \\
&= \Lambda_{k+1}^2 \left( \frac{N}{2^{k+1}} \right)^2 + \frac{1}{2}\lambda_{k+1}\Lambda_{k+1}(1 - \Lambda_k)\frac{N}{2^{k+1}} + \Lambda_{k+1}(1 - \lambda_{k+1})\frac{N}{2^{k+1}} \\
&= \left( \frac{N}{2^{k+1}} \Lambda_{k+1}\right)^2 + \frac{N}{2^{k+1}} \Lambda_{k+1} \left( \frac{1}{2}\lambda_{k+1}(1 - \Lambda_k) + 1 - \lambda_{k+1}\right) \\
&= \left( \frac{N}{2^{k+1}} \Lambda_{k+1}\right)^2 + \frac{N}{2^{k+1}} \Lambda_{k+1} \left( 1 - \frac{1}{2}(\lambda_{k+1} + \Lambda_{k+1}) \right) \\
&\leq \left( \frac{N}{2^{k+1}} \Lambda_{k+1}\right)^2 + \frac{N}{2^{k+1}} \Lambda_{k+1} \left( 1 - \Lambda_{k+1} \right),
\end{align*}
and thereby completing the induction argument.
\end{proof}

We now prove Proposition~\ref{thm:lblargep}. First, note that under Exponential Cliques, at each epoch we have two types of cliques: Type A consisting of all 0 workers and Type B consisting of a single 1 worker, the remaining workers being 0. Each 0 worker starts in a singleton clique of type A. There are two steps to each 0 worker being identified. It first becomes a part of a type B clique, and in the process gets matched to a 1 worker. In the second step, this type B clique gets matched to another type B clique. In this step, there are two possibilities: either the 0 worker gets matched to the lone 1 worker of the other clique, or before that happens, the two lone 1 workers in the two cliques get matched to each other, thus identifying everyone in the two cliques. In the first case, the 0 worker gets matched to a 1 exactly once more, while in the latter case it doesn't. The two possibilities are equiprobable. Thus in expectation, each 0 worker gets matched to $3/2$ 1 workers, leading to a regret of $3/4$ per 0 worker. Thus the total regret per worker is upper bounded by $3/4(1-p)$. 

But this argument assumes that there are no unpaired cliques at any epoch. If a 0 worker is a part of an unpaired clique of type $B$, then it gets matched to a $1$ worker in that clique more than once. The contribution of such a match to the overall regret can at most be $1/(2N)$ at every time step (since there is at most one 1 worker in an outlier clique at any time step). But over $k \approx \sqrt{2\log{N}}$ epochs, the total number of time steps is $\textup{o}(N)$. Thus the contribution to the regret from the outlier cliques throughout the run of Exponential cliques is $\textup{o}(1)$.  

What remains is the contribution to the regret from the workers left unidentified after the run of Exponential cliques, i.e., when there are randomly matched to each other. the rest of the proof shows that the contribution to regret from this phase is $\textup{o}(1)$ as well.

Let $W_k$ denote the number of unknown workers at the start of epoch $k$.  If we stop Exponential Cliques at epoch $k$ and afterward complete all remaining matches as fast as possible, even if all these matches incur regret, the total regret is $W_k(W_k-1)/2$ (which will take $W_k$ time steps). Thus, the expected per-worker regret incurred after stopping Exponential Cliques is at most $ E(W_k^2) / N$. We show that if we stop Exponential Cliques at epoch $K \approx \sqrt{2\log(N)}$, $\textup{E}(W_K^2)/N = \textup{o}(1)$. Note that this would also imply that $\textup{E}(W_K) = \textup{o}(N)$ and this the number of steps required to finish learning every workers type after EC stops is $\textup{o}(N)$. Thus, overall $E(\tau)=\textup{o}(N)$.

Let $C_k$ denote the number of unknown cliques at the start of epoch $k$. Then 
\begin{equation}\label{eq:WorkerToClique}
\textup{E}[W_k ^2]\leq  \textup{E}[2^{2k} C_k^2]. 
\end{equation}

Let $D_k$ be the number of cliques of size exactly $2^k$. When two unknown cliques are matched, then they give rise to an unknown clique in the next epoch if they are not both cliques of high type\footnote{A high type clique contains a worker of type 1.}. Therefore, we know that
\begin{align*}
D_{k} &\sim Bin(\lfloor \frac{D_{k-1}}{2} \rfloor, 1 - q_k^2), \forall k \in \{1,2,\ldots\},
\end{align*}
where $q_k = \frac{2^kp}{2^kp + 1-p}$.

Note that all the cliques are of size $2^k$, except possibly just one of smaller size. The presence of an unpaired clique of size $2^k$ or a clique with fewer than $2^k$ workers (and there can be only one such clique), will remain at the end of the epoch and that all future unpaired cliques will be merged with this one. Hence,
\begin{align*}
C_k \leq D_k + 1.               
\end{align*}

All that remains to be done is to show that 
\begin{align*}
\textup{E}[2^{2k} C_k^2] = 2^{2k} (\textup{E}[D_k^2]+2\textup{E}[D_k]+1) = o(N).
\end{align*}
Let $Q_k = \prod_{i=1}^{k-1}(1 - q_i^2)$. Invoking Lemma \ref{lemma:VarianceBound}, we know that
\begin{align*}
\textup{E}[D_k] &\leq \frac{N}{2^k} Q_k, \\
\textup{E}[D_k^2] &\leq \left( \frac{N}{2^k} Q_k\right)^2 + \frac{N}{2^k} Q_k(1 - Q_k) \leq \left( \frac{N}{2^k} Q_k\right)^2 + \frac{N}{2^k} Q_k.
\end{align*}
Utilizing these inequalities in the equation above, we now have to show that
\begin{align*}
2^{2k} \left( \left( \frac{N}{2^k} Q_k\right)^2 + 3\frac{N}{2^k} Q_k + 1 \right) = o(N),
\end{align*}
or equivalently,
\begin{align}\label{eqn:tally}
N^2 Q_k^2 + 3\cdot 2^k \cdot NQ_k + 2^{2k} = o(N),
\end{align}
Note that
\begin{align*}
Q_k &= \prod_{i=1}^k (1-q_i^2) 
        = \prod_{i=1}^k \frac{(1-p)(2^{i+1}p + 1 - p)}{(2^ip+1-p)^2} \leq \prod_{i=1}^k \frac{2^{i+1}p + 1}{2^{2i}p^2}\\
        & \leq  \prod_{i=1}^k \frac{2^{i+1}(1+ p) }{2^{2i}p^2} \leq \prod_{i=1}^k \frac{f(p)}{2^{i-1}} = 2^{-\frac{k^2}{2} + ( 1/2 - \log(f(p)))k}
        = 2^{-\frac{k^2}{2} + c(p)k},
\end{align*}
where $f(p)$ is a scalar dependent on $p$ alone, and so also $c(p)$.

Using the above, and the fact that $k \approx\sqrt{2\log{N}}$, we need to prove equation \ref{eqn:tally}.
Going over the terms on the right hand side of the same equation:
\begin{align*}
N^2 Q_k^2 &= N^2 2^{-k^2 + 2c(p)k} = N^2 2^{-2\log N} 2^{2c(p)\sqrt{2\log N}} = 2^{2c(p)\sqrt{2\log N}}, \\
3\cdot 2^k \cdot NQ_k &= 3\cdot  2^k N 2^{-\frac{k^2}{2} + c(p)k} = 2^{(1+c(p))\sqrt{2\log N}}, \\
2^{2k} &= 2^{2\sqrt{2\log N}}, \\
\end{align*}
each of which is clearly $\textup{o}(N)$.
\end{proof}

\begin{proof}[Proof of Proposition~\ref{thm:lblargep}]
Consider an optimal policy $\phi$. Recall that $N_1$ denotes the number of high quality workers. The platform must match at random at $t=0$. Let $K_1$ denote low quality workers matched to a high quality worker at this random matching stage. This means that $N_1 - K_1$ high quality workers were paired with another high quality worker, and were thus identified, and $K_1$ high quality workers are still unknown.  Let $K_2$ denote the number of low-quality workers that are matched to a high-quality worker at $t=0$ and a second high-quality worker at $t=1$. We show for the policy $\phi$ that $\E(K_2) \geq  N (p^2-p^3)/(1+p)$.  

First consider $\E(K_2 | K_1, N_1)$. An algorithm minimizes $K_2$ by pairing workers in the unknown graph amongst themselves and not with the known high type workers. Notice that the unknown population at $t=1$ consists of $K_1$ high quality workers and $N-N_1$ low quality workers. Since each of the $K_1$ low quality workers paired with a high quality worker at $t=0$ now has a $\frac{K_1}{N-(N_1-K_1) - 1}$ probability of getting matched to another high quality worker, 

 $$\E(K_2 | K_1, N_1) \geq K_1 \frac{K_1}{N-(N_1-K_1) - 1} $$.
 Using the fact that $\E(K_1|N_1) = (N-N_1) \frac{N_1}{N-1}$ and taking the expectation with respect to $K_1$, we get
 $$ \E(K_2|N_1) \geq \frac{[(N-N_1)(N_1)/(N-1)]^2}{N-N_1 - (N-N_1)(N_1)/(N-1) - 1}. $$

 Then $\E(N_1) = N\cdot p $ gives us 
 $$ \E(K_2) \geq N \cdot \frac{p^2-p^3}{1+p} + o(N). $$
 Now $E(K_1) = Np(1-p)$. Thus the total per worker regret is at least  $1/(2N)(\E(K_1)+\E(K_2))= (1/2)p^2(1-p)/(1+p) + (1/2)p(1-p) +\textup{o}(1)$, which simplifies to the expression in the statement. 
\end{proof}

\begin{proof}[Proof of Proposition~\ref{thm:ubec-l1}]
We calculate the expected decrease in regret by modifying the $k$-stopped Exponential Cliques algorithm to include one round of learning in the first epoch. 

Let $1, \ldots, M$ denote the workers included in cliques in the exploration set at this first epoch; the remaining workers $M+1, \ldots, N$ will proceed according to $k$-stopped Exponential Cliques.  For $i=1, \ldots, M$ let $X_i$ denote the per worker regret obtained  For the remaining workers $i=M+1, \ldots, N$ let $Y_i$ denote the regret obtained through $k$-stopped Exponential Cliques.

The index $K$ depends only on $N$ and the number of workers identified in the first time step. The regret thatt worker $i$ in the exploration set, and thus in the unknown graph at the beginning of the second epoch, incurs is independent of the realization of the workers in the first time step. Furthermore the $X_i$ are independent and identically distributed. Thus by finiteness of $\E[K]$ and $\E[X_i]$, Wald's equation concludes that 

\begin{equation}
\E[\sum_{i=1}^K X_i] = \E[K] \cdot \E[X_i] .
\end{equation}

Similarly, since the regret $Y_i$ incurred by any worker not in the exploration set is also independent from the number of workers revealed in the first time step, Wald's equation and our previous analysis of Exponential Cliques (without learning) gives

\begin{equation}
\E[\sum_{i=K+1}^N Y_i] = \E[N-K] \cdot \E[Y_i] = \E[N-K] \cdot \frac{3(1-p)}{4N}.
\end{equation}

Consider $\E[K]$. We number of pairs in the exploring set is $\min\{2W, \lfloor \frac{N-2W}{2}\rfloor\}$ where $W$ denotes the number of high type workers found by matching at random at $t=0$. Since $p < 1/3$, we are guaranteed that $\min\{2W, \lfloor \frac{N-2W}{2}\rfloor\} = 2W$. Then $\E[K] = 2 * \E[2W] = 2Np^2$.

Now we calculate $\E[X_i]$. Recall that distributed learning matches two unknown pairs together, so there are three possible cases.
\begin{enumerate}
\item Two high type cliques: In the first step, if both low type workers are learned by the known high types, there are 4 bad matches. If both high type workers are learned by the known workers, there are two bad matches. Otherwise, there is an average of 4 high-low matches made per low type worker. This gives an expected average of 7/4 high-low matches per low type worker an average regret of 7/8 per low type worker. 
\item One high type clique, one low type clique: The high type worker can be learned in the first or second step. If it is learned in the first, there are 3 bad matches, otherwise there are 4. This gives an expected average of 7/6 high-low matches per low type worker and an average regret of 7/12 per low type worker.
\item Two low type cliques: There are four matches made between high and low type workers, giving an average of 1 high-low match per low type worker and thus an average regret of 1/2 per low type worker. 
\end{enumerate}

Thus among the exploring set, the expected per person regret is 
\begin{align*} E[X_i] &= \frac{(2/4) (7 / 8)}{N} * P\big[(\theta^{(C_1)}, \theta^{(C_2)}) = (1,1)\big] \\
& + \frac{(1/2)}{N} *P\big[(\theta^{(C_1)}, \theta^{(C_2)}) = (0,0)\big] \\
&+ \frac{(3/4) (7 / 12)}{N} * P\big[(\theta^{(C_1)}, \theta^{(C_2)}) = (1,0)\big].\\
\end{align*}

At the beginning of the second epoch, the probabilities for a clique having a high-quality worker and not having such a worker are given by the following equations:

\begin{equation}\label{prob high type clique}
 \E[P(\theta^{(C_1)} = 1)] = \frac{2p}{1+p}.
\end{equation}

\begin{equation}\label{prob low type clique}
\E[P(\theta^{(C_1)} = 0)] = \frac{1-p}{1+p}.
\end{equation}

It then follows that for $p< 1/3$, $\E[X_i] = 3/2 - c$ for some $c$ that does not depend on $N$. Then (2) and (3) together imply that 

$$\E[\sum_{i=1}^K X_i + \sum_{i=K+1}^N Y_i] < \frac{3}{4}(1-p)  + 2p^2 [c - \frac{3(1-p)}{4}] < \frac{3}{4}(1-p).$$
where $c$ is a constant that does not depend on $N$. 

\end{proof}

\begin{proof}[Proof of Lemma~\ref{lma:char of max regret}.] 
For $p < 0.5$, the fraction of type 1 workers realized is less than $0.5$ with probability $1 - o(1)$. And whenever the fraction (say $s$) of realized type 1 workers is less than $0.5$, i.e., there are more type 0 workers than there are type 1 workers, the maximum payoff that an omniscient algorithm can generate in any given time step is $Ns$. This is obtained by allocating each type 1 worker to a different job, and the maximum number of jobs you can satisfy in this case is $Ns$. 

It is easy to see, for large enough $N$, that $\mathbb{E}[Ns] = Np - \textup{o}(N)$ (by a regular concentration argument). And therefore asymptotically, the payoff per worker (i.e., total payoff normalized by $N$) generated by an omniscient algorithm is $p - \textup{o}(1)$ per time step. Thus,

\begin{align*}
\mathbb{E} \left[ \frac{\tau}{N} \max_{A\in \Psi_N}\sum_{(i,j)\in A}P(\theta_i, \theta_j) \right] &= \frac{1}{N} \mathbb{E} \left[\sum_{t = 1}^\tau \left[ Np - \textup{o}(N) \right] \right] 
=  - \textup{o}(1) \mathbb{E}[\tau] + \mathbb{E} \left[\sum_{t = 1}^\tau Np\right] \\
&= \frac{1}{N} \mathbb{E} \left[\sum_{t = 1}^\tau Np\right] - \textup{o}(1).
\end{align*}
The above calculation holds because for the algorithms that we consider $\mathbb{E}(\tau)\leq\textup{O}(1)$.

For any fixed policy $\phi$, at time $t$, let $X_t$ the number of pairs that output a payoff of 1, and $Y_t$ the number of type 1 workers. We know that $\mathbb{E}Y_t = Np$ by our i.i.d. Bernoulli assumption. The asymptotic regret of $\phi$ can be calculated as follows:

\begin{align*}
\textup{Regret}_N(\bm{\theta},\phi, \tau) &= \frac{1}{N} \mathbb{E} \left[ \sum_{t = 1}^\tau  Np \right] - \frac{1}{N} \sum_{t=1}^\tau \sum_{(u,v) \in A(t)} P(\theta_u, \theta_v) - \textup{o}(1)\\
&= \frac{1}{N} \mathbb{E}\left[ \sum_{t = 1}^\tau Y_t - \sum_{t=1}^\tau X_t \right] = \frac{1}{N} \mathbb{E} \left[ \sum_{t = 1}^\tau (Y_t - X_t) \right]- \textup{o}(1).
\end{align*}

To interpret the term $Y_t - X_t$ in an interesting and useful way: observe that $X_t$ is total number of pairs that output a payoff of 1, i.e., the number of pairs (say $a_t^{0,1}$) of type $0-1$ (i.e., having just one worker of type 1), and the number of pairs (say $a_t^{1,1}$) of type $1-1$ (i.e., having two worker of type 1), whence $X_t = a_t^{0,1} + a_t^{1,1}$. And $Y_t$ is the number of workers of type 1, i.e., $Y_t = a_t^{0,1} + 2a_t^{1,1}$. $Y_t - X_t$ is then equal to $a_t^{1,1}$, or equivalently, the number of pairs of type $1-1$. Therefore: 
\begin{align}
\textup{Regret}_N(\theta,\phi, \tau) &= \frac{1}{N} \mathbb{E} \left[ \sum_{t = 1}^\tau  a_t^{1,1} \right]- \textup{o}(1)\\
&=\textup{Regret}''_N(\theta,\phi, T) - \textup{o}(1).
\end{align}

An analogous calculation (we omit details for the sake of brevity) implies that for $p > 0.5$, normalized regret can be calculated as: $\textup{Regret}_N(\theta,\phi, T) = \frac{1}{N} \mathbb{E} \left[ \sum_{t = 1}^\tau  a_t^{0,0} \right]$, where $a_t^{0,0}$ is the number of pairs of type $0-0$ at time $t$.
\end{proof}

\begin{proof}[Proof of Proposition~\ref{thm:2chain}.\ref{thm:2chain_a}]
Given two unknown pairs, i.e., of type either $0-1$ or $1-1$, we compare the performance of $1$-chain and $2$-chain learning on these. Let $q$ be the probability that an unknown pair is of type $0-1$. From a simple calculation we know that $q = \frac{2(1-p)}{2-p}$.

Let i-j and k-l be the two unknown pairs. Given a known $0-0$ pair, the 2-chain distributed learning algorithm tests the pairs 0-$i$, l-$0$, and $j-k$. Let the outputs of these pairs be A,B and C respectively.

There are four events $E_1$,$E_2$,$E_3$ and $E_4$:
\begin{enumerate}
\item $E1$: $A=B=1$ and $C=0$: In this case, we know immediately that $i=l=1$, and $j=k=0$, 
\item $E2$: $A=B=1$ and $C=1$, In this case, we know that $i=l=1$, and $j-k$ could be of type either $0-1$ or $1-1$. 
\item $E3$: $A=B=0$ and $C=1$:  In this case, we know that $i=l=0$, and $j=k=1$. 
\item $E4$: $A=1,B=0$ or $A=0,B=1$ and $C=1$: In this case, either $i=1, l=0$ or $i=0, l=1$. When $i=1, l=0$, we know that $k=1$, since $k-l$ is of type $0-1$ or $1-1$ to begin with, and $j$ could be either 0 or 1. The analogous observation holds when $i=0, l=1$.  
\end{enumerate}

 Imagine running the 1-chain learning policy on some instance. Recall that 1-chain learning takes a known $0-0$ pair, and an unknown pair, say $i-j$ which is known to be of type $0-1$ or $1-1$, and tests the pairs 0-$i$ and $j$-0. Let's say there is another unknown pair $k-l$ (of type $0-1$ or $1-1$) that is retained as is by the 1-chain policy in the current step, and in the next step is tested across a known $0-0$ pair (like $i-j$ was previously). Hypothetically, do a 2-chain replacement on $i-j$ and $k-l$ taken together, where in the case of event $E_2$ we test the pairs 0-$j$ and $k$-0 in the next time step (i.e., mimicking the 1-chain policy). We first calculate how this replacement affects the accumulation of regret. We do this by conditioning on each of the events $E1$ through $E4$. 
 
 Recall that regret can be computed by looking at just the $0-0$ pairs formed, by Lemma \ref{lma:char of max regret}. In each of these events, denote the regret incurred by 1-chain learning and 2-chain learning as $X_i$ and $Y_i$ respectively. We now describe how we compare the regret accrued by the two policies.

Under event $E_1$, since one only $0-0$ pair ($j-k$) is formed by 2-chain Learning, we have $Y_1 = 1$. And since 1-chain learning forms two $0-0$ pairs, one with $i-j$, and one with $k-l$, it incurs a regret $X_1 = 2$. 

Under event $E_2$, no $0-0$ pair is formed by the 2-chain replacement. In this case, $j-k$ is of type $0-1$ or $1-1$ with probability $q$ and $1-q$ respectively, and therefore is indistinguishable from any other unknown pair of type $0-1$ and $1-1$. As such the 2-chain replacement block tests $j-k$ in the next time step with 0-$j$ and $k$-0 pairs, and the regret incurred here is equal to $q$. Therefore, $Y_2 = 1-q$. Of the two pairs $i-j$ and $k-l$, one of the two is guaranteed to be $1-1$, and the other is of type $0-1$ with probability $q = \frac{2(1-p)}{2-p}$. In the case both are of type $1-1$, we have no $0-0$ pairs formed by doing the 1-chain policy. In the case that one of them is of type $0-1$, we have one $0-0$ pair formed by the 1-chain policy. Since the latter event occurs with probability $q$, we have that $X_2 = q$.

Under event $E_3$, $Y_3 = 2$, since two $0-0$ pairs are formed in the 2-chain replacement. And $X_3 = 2$ since we have two $0-1$ type pairs, and the 1-chain makes two $0-0$ pairs in learning them.  

Under event $E_4$, 2-chain learning makes one $0-0$ pair. Also, one of $i,l$ is 0 and the other is 1. Let's assume $i=1$ and $l=0$, which implies that $k=1$. $j$ can be of either type, and by our i.i.d. Bernoulli model, is 1 with probability $p$ and $0$ otherwise. The 2-chain replacement next makes the pair $0-j$, and this is $0-0$ with probability $1-p$ - there is regret here only when $j=0$ . Therefore, we have $Y_4 = 1 + (1-p) = 2-p$. For the 1-chain policy, since $k-l$ is of type $1-1$, and $j-k$ is of type $0-1$ with probability $1-p$, the regret incurred is also $X_4 = 1 + (1-p) = 2-p$. 

We see that $Y_1 > X_1$ and $Y_2=X_2,Y_3=X_3,Y_4=X_4$. Taking the expectation over all the events, we have that the 2-chain replacement lowers regret when applied to the 1-chain policy. By doing repeated 2-chain replacements, we get that the 2-chain policy has lower regret than the 1-chain policy. Note that our 2-chain replacement (on $i-j$ and $k-l$) mimics the 1-chain policy on $j-k$ in the next step in the case of event $E_2$. However, since the pair $j-k$ has the same distribution as any other pair that is known to be of type $0-1$ or $1-1$, we can do a repeated application of our 2-chain replacement if needed.
\end{proof}

\begin{proof}[Proof of Proposition~\ref{thm:2chain}.\ref{thm:2chain_b}]
Given two unknown pairs, i.e., of type either $0-1$ or $1-1$, we compare the performance of $1$-chain and 2-chain learning on these. Let $q$ be the probability that an unknown pair is of type $0-1$. From a simple calculation we know that $q = \frac{2(1-p)}{2-p}$.

Let i-j and k-l be the two unknown pairs. Given a known $0-0$ pair, the 2-chain distributed learning algorithm tests the pairs 0-$i$, $l$-0, and $j-k$. Let the outputs of these pairs be A,B and C respectively.

There are four events $E_1$,$E_2$,$E_3$ and $E_4$:
\begin{enumerate}
\item $E1$: $A=B=1$ and $C=0$: In this case, we know immediately that $i=l=1$, and $j=k=0$. Also, $i-j$ and $k-l$ both have to be of type $0-1$. Since, in $i-j$ both $i$ and $j$ are equally likely to be 0, and similarly in $k-l$, we have$\mathrm{Pr}[E1] = \frac{q^2}{4} = \frac{(1-p)^2}{(2-p)^2}$
\item $E2$: $A=B=1$ and $C=1$, In this case, we know that $i=l=1$, and $j-k$ could be of type either $0-1$ or $1-1$. In this case, both $i-j$ and $k-l$ could be of type $1-1$, or one of them is $1-1$ and the other is $0-1$, we have $\mathrm{Pr}[E2] = (1-q)^2 + q(1-q) = \frac{p(2-p)}{(2-p)^2}$.
\item $E3$: $A=B=0$ and $C=1$:  In this case, we know that $i=l=0$, and $j=k=1$. By a calculation similar to the $E_1$ case, we have $\mathrm{Pr}[E3] = \frac{q^2}{4} = \frac{(1-p)^2}{(2-p)^2}$.
\item $E4$: $A=1,B=0$ or $A=0,B=1$ and $C=1$: In this case, either $i=1, l=0$ or $i=0, l=1$. When $i=1, l=0$, we know that $k=1$, since $k-l$ is of type $0-1$ or $1-1$ to begin with, and $j$ could be either 0 or 1. The analogous observation holds when $i=0, l=1$.  Again by a calculation similar to that for $E_2$, we have $\mathrm{Pr}[E4] = q(1-q) = \frac{2(1-p)}{(2-p)^2}$.
\end{enumerate}

 Imagine running the 2-chain learning policy on some instance. In the case of event $E_2$ being observed by the policy, assume that an oracle then reveals the types of $j$ and $k$ (after $j-k$ have been paired once). As a result, from then on, there is no regret involving either $j$ and $k$. The access to such an oracle can only decrease the regret incurred by the 2-chain policy. Let's say the two pairs $i-j$ and $k-l$ are processed as a 2-chain block. Do a 1-chain replacement using $i-j$ and $k-l$, whereby each of these are tested with two known 0 type agents in sequence. Without loss of generality, this involves making the pairs $0-i$, $0-j$ and $k-l$ first, and then since we will have learned the types of $i$ and $j$, we do $0-k$ and $0-l$. We now calculate how this replacement affects the accumulation of regret. We do this by conditioning on each of the events $E1$ through $E4$. 
 
 Recall that regret can be computed by looking at just the $1-1$ pairs formed, by Lemma \ref{lma:char of max regret}. In each of these events, denote the regret incurred by 1-chain learning and 2-chain learning as $X_i$ and $Y_i$ respectively. We now describe how we compare the regret accrued by the two policies.

 Under event $E_1$, since both $i-j$ and $k-l$ are of type $0-1$, no $1-1$ pair is formed by either the 2-chain policy or the 1-chain replacement. Therefore, we have $X_1 = Y_1 = 0$.
 
 Under event $E_2$, $Y_2$ is just equal to the probability that $j-k$ is a $1-1$ pair, which is $1-q$. Therefore we have $Y_2 = \frac{p}{2-p}$. Since we invoke an oracle that reveals the types of $j$ and $k$, after they have been paired once, there is no more regret involving $j$ or $k$.
 
Also, under event $E_2$, if $j-k$ is of type $1-1$, then the 1-chain replacement has a regret of $1$ since in this case both $i-j$ and $k-l$ are of type $1-1$, and the 1-chain replacement has to retain one of them in it's first step. Therefore $\textup{E}[X_2|j-k \mbox{ is of type $1-1$}] = 1$. If $j-k$ is of type $0-1$, then we have out of $i-j$ and $k-l$, one of type $0-1$ and the other of type $1-1$. , with both possiblities (i.e., $i-j$ being of type $0-1$ and $k-l$ of type $1-1$, and vice versa) being equally likely. The 1-chain replacement retains $k-l$ in it's first step, and $i-j$ is equally likely to be  $0-1$ or $1-1$. If $i-j$ is $0-1$, then $k-l$ is $1-1$, and the regret incurred is equal to $1$ (arising due to $k-l$). If $i-j$ is $1-1$, then $k-l$ is $0-1$, and there is no regret, which gives us $\textup{E}[X_2|j-k \mbox{ is of type $0-1$}] = 0.5(1) + 0.5(0) = 0.5$. Therefore $$X_2 = (1-q)\textup{E}[X_2|j-k \mbox{ is of type $1-1$}]+q\textup{E}[X_2|j-k \mbox{ is of type $0-1$}] = 1-q + 0.5q = \frac{1}{2-p}.$$

Under event $E_3$, we have $i=l=0$ and this necessarily means that $j=k=1$, and therefore $Y_3 = 1$ since $j-k$ is of type $1-1$. And also, $X_3 = 0$ since no $1-1$ pairs are formed.

Under event $E_4$, one of $i,l$ is 0 and the other is 1, both configurations being equally likely. First assume $i=1$ and $l=0$. This implies that $k=1$. $j$ can be of either type, and by our i.i.d. Bernoulli model, is 1 with probability $p$ and $0$ otherwise. In the case that $j=1$, we have one $1-1$ pair matched by 2-chain learning. Hence $\textup{E}[Y_4|i=1,l=0] = p$, and similarly $\textup{E}[Y_4|i=1,l=0] = p$, whence $Y_4 = p$. In this case, as before, 1-chain replacement retains $k-l$ in the first step. If $i=1,l=0$, then $k-l$ is $0-1$ and there is no regret. If $i=0,l=1$, then $k$ is 1 with probability $p$ and 0 otherwise. Therefore, $k-l$ is $1-1$ with probability $p$ and consequently, and $X_4 = 0.5p$.

The expected regret of the 2-chain block corresponding to $i-j$ and $k-l$ (with the oracle being invoked in case of event $E_2$) is
\begin{align*}
\sum_{i}\mathrm{Pr}[E1]Y_1 = \frac{p(2-p)}{(2-p)^2}(\frac{p}{2-p}) + \frac{(1-p)^2}{(2-p)^2}(1) + \frac{2(1-p)}{(2-p)^2}(p) = \frac{1}{(2-p)^2}.
\end{align*}

The expected regret of the 1-chain replacement can be computed to be 
\begin{align*}
\sum_{i}\mathrm{Pr}[E1]X_1 = \frac{p(2-p)}{(2-p)^2}(\frac{1}{2-p}) + \frac{2(1-p)}{(2-p)^2}(0.5p) = \frac{p}{2-p}.
\end{align*}

We see clearly that $\sum_{i}\mathrm{Pr}[E1]X_1 < \sum_{i}\mathrm{Pr}[E1]Y_1$, and by invoking a repeated replacement argument analogous to the proof of Theorem \ref{thm:2chain}.\ref{thm:2chain_a} we have that the 1-chain policy outperforms the 2-chain policy. 
\end{proof}

\end{document}